\documentclass{article}

\PassOptionsToPackage{numbers}{natbib}
\usepackage[preprint]{neurips_2026}

\usepackage[utf8]{inputenc} 
\usepackage[T1]{fontenc}    
\usepackage{hyperref}       
\usepackage{url}            
\usepackage{booktabs}       
\usepackage{amsfonts}       
\usepackage{nicefrac}       
\usepackage{microtype}      
\usepackage{xcolor}         
\usepackage{graphicx}
\usepackage{bm}
\usepackage{algorithm}
\usepackage{algpseudocode}
\usepackage{float}
\usepackage{enumitem}
\usepackage{amsmath,amssymb,amsfonts}
\usepackage{amsthm}
\usepackage{mathtools}
\usepackage{placeins}
\usepackage{cancel}
\usepackage{subcaption}
\usepackage{booktabs}
\usepackage{makecell}
\usepackage{hyperref}
\usepackage{url}

\usepackage{natbib}

\hypersetup{
   colorlinks=true,
   citecolor=black
}

\newtheorem{theorem}{Theorem}

\newtheorem{lemma}{Lemma}
\newtheorem{corollary}{Corollary}

\newcommand{\Ans}{\mathrm{Ans}}
\newcommand{\Pool}{\mathcal{P}}
\newcommand{\Cand}{\mathcal{A}}
\newcommand{\New}{\mathcal{B}}
\newcommand{\Subpool}{\mathcal{D}}
\newcommand{\Xspace}{\mathcal{X}}
\newcommand{\length}{\mathrm{len}}
\newcommand{\pa}{\mathrm{pa}}
\newcommand{\ind}{\mathbf{1}}

\newcommand{\PRM}{\mathrm{PRM}}

\newcommand{\bbE}{\mathbb{E}}
\newcommand{\dd}{\,\mathrm{d}}

\usepackage{graphicx}
\usepackage[table]{xcolor}
\usepackage{hhline}

\usepackage{cleveref}
\usepackage{booktabs}
\usepackage{tabularx}
\usepackage{amsmath}
\usepackage{amssymb}
\usepackage{float}
\usepackage{etoc}

\DeclareMathOperator{\Top}{Top}
\newcommand{\TopPool}[3]{\Top\!\left(#1\,;\,#2\,;\,#3\right)}

\title{Beyond the Frontier: Stochastic Backtracking for Efficient Test-Time Scaling}

\author{%
Dao Tran\thanks{Contributed equally} , Duc Anh Le\footnotemark[1] , 
Ngoc Luu, Quan Pham, Tung Pham, Hung Bui  \\
\texttt{ \{tranao, lducanh, nluu, quanpham, tungp, hungbui\}}  \\
\texttt{@qti.qualcomm.com} \\[0.7em]
Qualcomm AI Research\thanks{Qualcomm AI Research is an initiative of Qualcomm Technologies, Inc.}\\[-0.1em]
}

\definecolor{deltaPass}{RGB}{230, 126, 34}  
\definecolor{deltaTok}{RGB}{52, 152, 219}   

\newcommand{\sd}[2]{#1$_{\pm#2}$}

\renewcommand{\arraystretch}{1.08}

\begin{document}

\maketitle
\begin{abstract}

Test-time scaling improves language-model reasoning by spending additional compute to explore multiple solution trajectories. The key challenge is to maximize accuracy while minimizing the total number of generated tokens during reasoning. Recent PRM-guided methods score intermediate prefixes to steer this search, but most are frontier-only: they keep only the current active prefixes and irreversibly prune or resample away the rest using noisy PRM scores. This can cause premature commitment, diversity collapse, and the loss of prefixes that still admit correct continuations. We introduce \emph{stochastic backtracking over a persistent pool} of historical prefixes, allowing test-time compute to revisit previously generated states instead of only expanding the current frontier. To make this efficient, we propose two complementary mechanisms. \emph{Subpool Selection} strengthens greedy PRM-guided search by applying Top-$N$ selection within random subpools, giving historical prefixes a chance to bypass over-scored frontier candidates. \emph{Power Backtrack Sequential Monte Carlo} extends SMC-style resampling to the persistent pool using powered PRM scores and mixture-corrected weights. Across mathematical reasoning benchmarks and model scales, our methods consistently achieve higher accuracy per token count, and the same level of accuracy using only a fraction of the token count in comparison to strong PRM-guided baselines, demonstrating that persistent-pool stochastic backtracking provides a simple and effective way to improve the accuracy–token trade-off in test-time scaling.

\end{abstract}



\vspace{-1mm}
\section{Introduction}
\vspace{-1mm}
Test-time scaling (TTS) \cite{s1, jaech2024openai, guo2025deepseek} improves LLM reasoning by spending additional inference-time compute to explore, verify, or refine reasoning trajectories. This paradigm is especially attractive for mathematical reasoning, where a fixed model may fail under one decoding path but succeed when allowed to explore alternatives \cite{liu2025can, brown2024large}. Recent works \cite{snell2025scaling, karan2026reasoning} have shown that such compute can substantially improve reasoning performance, but these gains depend on allocating additional tokens effectively rather than spending them on unproductive paths.

A common way to allocate test-time compute is to use process reward models (PRMs) \cite{zhang2025lessons, wang2024math, yuan2025free, luo2024improve} to score prefixes, i.e., partial answers. Unlike outcome reward models \cite{stiennon2020learning, ouyang2022training, hosseini2024vstar, zhang2025generative}, which score only complete solutions, PRMs provide step-level signals that allow beam search \cite{snell2025scaling}, tree search \cite{yao2023tree}, and sequential Monte Carlo methods \cite{zhao2024probabilistic, feng2025stepbystep, puri2025probabilistic, giannone2025mitigating} to prioritize promising partial trajectories before a full solution is generated. This makes PRM-guided search more adaptive than Best-of-\(N\) \cite{brown2024large} or self-consistency \cite{wang2023selfconsistency}, which primarily operate over independently generated complete samples. However, PRM scores are imperfect proxies for downstream correctness: they can be noisy, locally ambiguous, and mismatched to the continuation behavior of the deployed generator \cite{zhang2025lessons, park2026know, zheng2025cold}.

This imperfection is especially damaging in frontier-only search. Many PRM-guided test-time scaling methods form the next search state only from newly generated children, discarding other previously generated prefixes. This frontier-only memory structure turns noisy intermediate scores into irreversible decisions. An under-scored prefix can be removed from future expansion even if it still admits a correct continuation, while an over-scored incorrect prefix can continue to receive compute. Thus, the failure mode is not merely that PRMs are imperfect, but that frontier-only search makes local errors permanent.

We propose \textit{stochastic backtracking over a persistent pool} 
of historical prefixes. Instead of restricting the next search state to newly generated children, the search keeps previously generated prefixes eligible for future expansion. Importantly, persistent pooling does not store the full \textit{exponentially} growing search tree. Each expansion round generates a fixed budget of \(N\) new prefixes, so after \(t\) rounds the full persistent pool contains at most \( Nt\) prefixes. Persistent pooling is therefore a linear-memory alternative to irreversible frontier-only search.


Persistence alone, however, is not sufficient. Greedy selection over the full historical pool can still be dominated by a few over-scored incorrect prefixes, while uniform sampling from history wastes tokens on stale or low-quality states. We therefore introduce two stochastic selection rules that preserve access to underestimated prefixes while retaining PRM guidance. \textit{Subpool Selection} strengthens greedy PRM-guided search by selecting top candidates within random historical subpools, allowing prefixes blocked by frontier nodes to be reconsidered. \textit{Power Backtrack SMC} extends SMC-style sampling to sequential multiple-importance-sampling over a persistent pool using powered PRM scores and mixture-corrected weights, enabling stochastic recovery from earlier pruning decisions.

\vspace{-2mm}

\paragraph{Contributions.}
We introduce stochastic backtracking over a persistent pool of historical prefixes for token-efficient test-time scaling. We instantiate this framework with two complementary methods: Subpool Selection, which ranks candidates within random shortlists to bypass over-scored blockers, and Power Backtrack SMC, which extends SMC-style sampling to sequential multiple-importance-sampling over historical prefixes using powered PRM scores and mixture-corrected weights. Across mathematical reasoning benchmarks and model scales, our methods improve the accuracy--token trade-off over strong PRM-guided baselines, showing that historical access reduces premature commitment in test-time scaling. 

\vspace{-1mm}
\section{Related Work}

\vspace{-1mm}
\paragraph{Test-Time Scaling.} Test-time scaling improves LLM performance by increasing inference-time computation, typically through either parallel or sequential scaling. Parallel scaling generates multiple independent candidates and selects or aggregates among them, as in self-consistency \cite{wang2023selfconsistency,nguyen2024consistent, chen2024more}, Best-of-N \cite{brown2024large, lightman2024lets, snell2025scaling}, and related sampling-based approaches \cite{wang2025learning, macfarlane2025instilling}. In contrast, sequential scaling adaptively spends compute within a reasoning process, extending or revising trajectories through longer deliberation or instance-dependent compute \cite{jaech2024openai,guo2025deepseek,s1,park2025instance}. Hybrid scaling combines these modes by maintaining broad exploration across alternatives while sequentially expanding promising partial trajectories, including beam and lookahead search \cite{snell2025scaling}, Tree of Thoughts or Graph of Thoughts \cite{yao2023tree, besta2024graph}, and reward-filtered search procedures \cite{snell2025scaling,yao2023tree,besta2024graph,yu2025limits,li2025cascade}. Our method belongs to this hybrid family, but focuses on how the search state is retained and revisited under a token-efficiency objective. 

\vspace{-2mm}

\paragraph{Sequential Monte Carlo.} SMC methods maintain a weighted population of particles that are iteratively propagated, weighted, and resampled to approximate complex posterior distributions, making them a natural fit for sequential conditional generation tasks \cite{zhao2024probabilistic}. 
For language models, twisted SMC constructs learned proposal distributions that reshape generation toward high-reward completions \cite{zhao2024probabilistic}, with direct application to step-level resampling in mathematical reasoning \cite{feng2025stepbystep, puri2025probabilistic}. Related work explores controlled generation via posterior inference \cite{xefteri2025syntactic}, adaptive rejection sampling \cite{lipkin2025fast}, model ensembling \cite{chan2026ensembling}, and SMC-based steering \cite{lew2023sequential, grand2025selfsteering,loula2025syntactic}. 
Entropic Particle Filter~\cite{giannone2025mitigating} proposed to anneal the resampling weights of the particle filter, therefore regularize the resampling over the current frontier to avoid premature exploitation in test-time scaling; whereas our PB-SMC changes the support of resampling by keeping historical prefixes eligible across rounds. 

The persistent pool in our backtracking SMC method is motivated by Persistent Sampling \cite{karamanis2025persistent}, originally proposed for Bayesian parameter inference. In order to reduce the impoverishment of the set of parameter-represented particles overtime, persistent sampling allows particles from previous iterations to persist in a growing weighted ensemble and resampling from the mixture of past distributions using multiple-importance sampling~\cite{veach1995optimally}. We first extend persistent sampling framework to the case of sequential non-Markovian state space inference given a terminal likelihood, and then show its effective application for test-time scaling in LLM reasoning.


\vspace{-2mm}

\paragraph{Backtracking in LLMs.} Several works train or prompt models to revise their own outputs, either through iterative self-correction or refinement of completed solutions \cite{madaan2023selfrefine, qu2024recursive, havrilla2024glore, kumar2025training}. These methods learn a corrective policy over one or a few solution attempts, whereas our method keeps the base generator fixed and changes the test-time search state. Other work studies explicit backtracking as sequential search in language, showing that its benefits depend strongly on task structure and training procedure \cite{gandhi2024stream, qin2025to}. Closest to our setting are verifier-guided backtracking methods that use process feedback to revise partial generations: preemptive backtracking identifies a problematic step in a single trace and resamples its suffix \cite{singh2025improving}, while VGB \cite{rohatgi2026taming} treats generation as a verifier-guided random walk on a tree with probabilistic backtracking. 
In contrast, we maintain a persistent pool of historical prefixes across PRM-guided search and stochastically choose which prefixes to expand next, focusing on history-aware compute allocation rather than self-correction within a trajectory. 

Our Subpool Selection follows the selective-candidate-set idea in SeeA* \cite{zhao2024seea}, which samples a subset of all nodes and expands the best node within the subset rather than the globally best node; in our setting, the set of nodes becomes a persistent pool of PRM-scored reasoning prefixes, and the sampled subset helps greedy selection to bypass over-scored blockers under noisy PRM guidance.

\vspace{-1mm}
\section{Problem Setup and Frontier-Only Failure}
\vspace{-1mm}
\label{sec:prelim-smc}
In this section, we present the problem setup, the mechanism of known TTS methods used in this work and discuss the problem of irreversible frontier-only selection and noisy PRM score.  
\vspace{-1mm}
\subsection{Prefix, traces and PRM scores}
\vspace{-1mm}
\label{subsec:prefix-trace-prm}
Given a prompt $x_0$, let $x_t$ denote the $t$-th reasoning step generated by LLMs. Denote $x_{1:t} = (x_1,\ldots,x_t)$ be a prefix, 
the probability of the prefix is calculated by the LLM as follows 
$$p_{\mathrm{LLM}}(x_{1:t}\mid x_0)=\prod_{s=1}^t p_{\mathrm{LLM}}(x_s\mid x_0,x_{1:s-1}).$$ 
For simplicity,  we  suppress  $x_0$ from formulas. Let $x_{1:T}$ denote a complete trace with extracted answer $\Ans(x_{1:T})$ and $\phi(\Ans(x_{1:T}))\in\{0,1\}$ be the correctness indicator of the answer $x_{1:T}$.  The target is to estimate the density of correct answers
\vspace{-0.5mm}
\begin{equation} 
    \sigma(x_{1:T})
    \propto
    p_{\mathrm{LLM}}(x_{1:T})\,\phi(\Ans(x_{1:T})).
    \label{eq:true-terminal-target-main}
\end{equation}
Let  $\sigma(x_{1:t}) = \sum_{x_{t+1:T}} \sigma(x_{1:T})$.  
 If the marginal $\sigma(x_{1:t})$ is positive, then the prefix $x_{1:t}$ is correct up until time $t$ and we can roll out to obtain a final correct answer. Otherwise if it equals zero, then all of the continuation from $x_{1:t}$ will produce wrong answer. Subsequently, we also use the short notation $z$ to refer to a prefix when the context is clear.



In practice, PRMs such as Qwen2.5-Math-PRM-7B \cite{zhang2025lessons} are trained on binary labels rather than on the probability that the deployed base LLM will complete a prefix correctly. Therefore, the PRM score \(r_{\mathrm{PRM}}(x_{1:t})\) should be viewed as a learned proxy for prefix correctness, sensitive to the PRM labeling procedure, not as the ideal prefix target \(\sigma(x_{1:t})=\sum_{x_{t+1:T}}\sigma(x_{1:T})\) or as a calibrated value function for the generator~\cite{wang2024making, luo2025improve}. This proxy mismatch can affect test-time search.  A prefix may receive a high PRM score because it can be completed by the stronger rollout policy used during PRM supervision, even though the deployed base LLM is unlikely to complete it correctly.  Conversely, a prefix that the base LLM could complete may be under-scored because of rollout variance, annotation noise, or PRM generalization error.  Frontier-only pruning is therefore brittle: it can turn a local mismatch between PRM score and the ideal prefix target into an irreversible search decision.

\vspace{-1mm}
\subsection{PRM-guided search as prefix selection}
\label{subsec:prm-search-prefix-selection}
\vspace{-1mm}

The Best-of-$N$ method samples independent complete traces and make decision based on  the PRM's evaluation only on terminal candidates
 . The Beam search keeps a frontier of prefixes, expands selected parents, and deterministically retains the top-scoring children. Meanwhile, at step $t$, the SMC has in hand a pool of  weighted frontiers $\mathcal{P}_{t-1}$, from the previous step. The SMC samples from this pool to obtain a multiset and use LLM to extend it to obtain a new set of prefixes denote by $\mathcal{B}_t$. This multiset with normalized weights will become the new pool $\mathcal{P}_t$ for the next step. 

\vspace{-1mm}
\subsection{Irreversible frontier-only selection}
\vspace{-1mm}
\label{subsec:frontier-only-failure}

The above-mentioned methods have a weakness that is
if a prefix $z$ is not in the frontier $\New_t$, then it will be  absent in next pool $\Pool_t$. Consequently, it will be never considered again. However, because the PRM could be noisy, thus its score for correct prefix could lower compared to that of incorrect prefixes, there is a chance that prefix is missed out due to its low score. Then  the loss of missing a correct answer is incurred permanently just because the mechanism of frontier-only selection. To solve this problem, we introduce the persistent-pool and backtracking mechanisms in the next section.

\vspace{-1mm}
\section{Persistent-Pool Stochastic Backtracking}
\label{sec:methodology}
\vspace{-1mm}


\subsection{Master Algorithm}
\label{subsec:master-algorithm}
\vspace{-1mm}

At step $t$, the algorithm maintains a multiset $\Pool_{t-1}$ of prefixes, PRM scores $r_{\mathrm{PRM}}(z)$, and selection weights $W_{t-1}(z)\ge0$ for any $z\in \mathcal{P}_{t-1}$.  A selection rule $f_t$ chooses a multiset of $M$ parents $\Cand_t=f_t(\Pool_{t-1};W_{t-1})$.  Each selected parent is expanded by the base LLM to produce $B=N/M$ children, so there will be exactly $N$ new prefixes; we denote this child multiset by $\New_t$.  An update rule $G_t$ is to construct a new pool of prefixes $(\Pool_t,W_t)=G_t(\Pool_{t-1},\New_t)$ with new weights $W_t$.
Since the new pool $\mathcal{P}_t$ could consist of prefixes from pool $\mathcal{P}_{t-1}$ of the previous step, we name it the \emph{persistent pool}. In our proposed algorithms 
Subpool Selection and backtracking SMC, we use different mechanisms to build the persistent pool.

\vspace{-2mm}

\begin{algorithm}[H]
\caption{Persistent-Pool Master Algorithm}
\label{alg:master-particle-search}
\begin{algorithmic}[1]
\Require prompt $x_0$, child budget $N$, parent budget $M$, horizon $T$, selection rules $f_t$, memory updates $G_t$.
\State Initialize $\Pool_0$ with $N$ one-step samples from $p_{\mathrm{LLM}}(\cdot\mid x_0)$ and set $W_0(z)\gets r_{\mathrm{PRM}}(z)$ for all $z\in\Pool_0$.
\For{$t=1,\ldots,T$}
    \State Select parents $\Cand_t\gets f_t(\Pool_{t-1};W_{t-1})$ as a multiset of size $M$.
    \State Generate $\New_t\gets\bigcup_{c\in\Cand_t}\textsc{LLM-Expand}(c,B)$ with $B=N/M$ children per parent.
    \State For every $z\in\New_t$ compute its PRM score $r_{\mathrm{PRM}}(z)$.
    \State Update memory and weights: $\Pool_t, W_t \gets G_t(\Pool_{t-1},\New_t)$.
\EndFor
\State Return final candidates from $\Pool_T$ by terminal ranking.
\end{algorithmic}
\end{algorithm}





\vspace{-2mm}

\begin{figure*}[t]
    \centering
    \includegraphics[width=\textwidth]{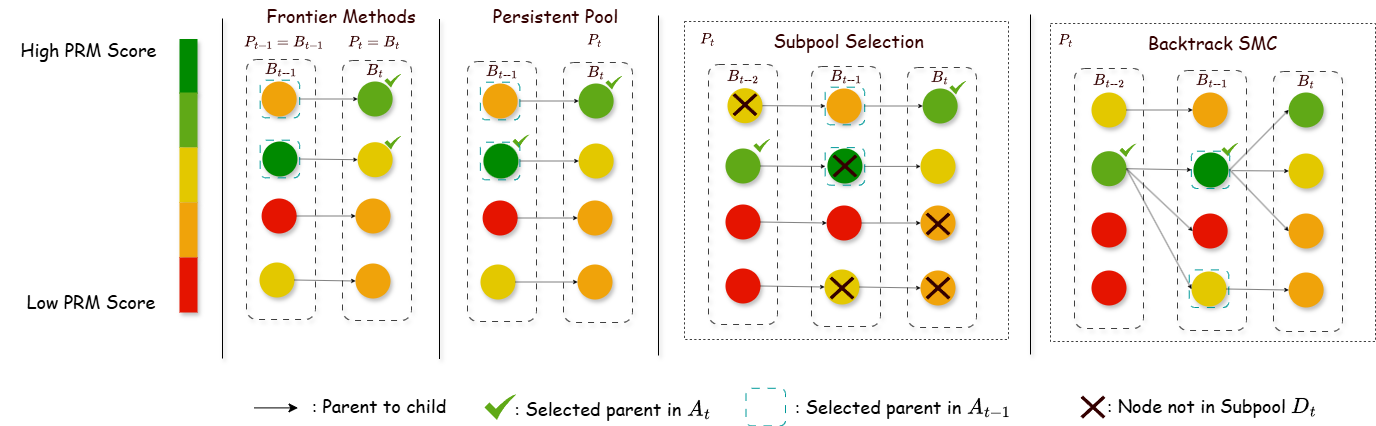}
    \caption{
    Schematic comparison of frontier-only search and persistent-pool backtracking.
    Frontier-only methods keep only the newly generated children $\New_t$, so prefixes not selected into the next frontier cannot be revisited.
    Greedy Selection keeps a persistent pool and applies global top ranking, but older correct prefixes are blocked by (always scored lower than) prefixes already selected at $t-1$.
    SPS resolves this by ranking only inside a random subpool, giving blocked historical prefixes another chance to be expanded.
    Power Backtrack SMC keeps historical prefixes through weighted resampling, combining a retained historical multiset $S_t$ with newly generated children $\New_t$.
    }
    \label{fig:persistent-pool-schematic}
\end{figure*}
\vspace{-3mm}
Figure~\ref{fig:persistent-pool-schematic} summarizes the key idea: stochastic backtracking preserves PRM guidance, but changes the memory and selection mechanism so that historical prefixes remain eligible for future compute.

To keep peak memory bounded, we store KV caches only for prefixes currently being expanded.  At step $t$, we maintain KV states only for the generated children $\New_t$.  Historical prefixes in $\Pool_t$ are kept as text states and PRM scores, not as live KV states.  When the search backtracks to an older prefix $z=x_{1:i}$, we recompute its KV state before expansion.  Thus backtracking adds a small compute overhead, but peak KV memory remains bounded by the $N$ rather than by $|\Pool_t|$. Our experiment results in~\ref{sec:acc_budget} show that the compute overhead is negligible compared to the gain in accuracy.

A persistent pool will grow with time and the low-quality prefixes could remain in the pool because of the selectors. For example, a naive global top $M$ selection rule will keep choosing the over-scored incorrect prefixes overtime, meanwhile a uniform sampling method could select low-score incorrect prefixes. These observations motivate us to introduce two types of algorithms based on adjusting searching and sampling. More particular, we introduce 
Subpool Selection to search for potential correct prefixes (Section~\ref{subsec:sps-method}) . For sampling direction, we introduce Power Backtrack SMC by sampling with replacement from a powered, mixture-corrected persistent pool (Section~\ref{subsec:power-backtrack-smc}).      



\vspace{-1mm}
\subsection{Instantiations of the master algorithm}
\label{subsec:template-instantiations}
\vspace{-1mm}

Algorithm~\ref{alg:master-particle-search} is a \textsl{generalized} algorithm covering both frontier-only baselines and our persistent-pool variants. In Table~\ref{tab:master-instantiations}, the selection rule $\mathrm{Top}(\mathcal{P}_{t-1},W_{t-1}, M) $ means we choose top $M$ elements from the pool $\mathcal{P}_{t-1}$ based on the associated weights $W_{t-1}$, and sample with replacement with $M = N$ means sampling from multinominal distribution built on pool $\mathcal{P}_{t-1}$ with normalized weights $W_t$. 
In our PB-SMC, we set $M=N$: each selected parent generates one child, matching the standard resample-and-propagate particle update. 
For a newly generated prefix $z\in\New_t$, let $\pa(z)$ denote the parent prefix of $z$.

\begin{table}[H]
\vspace{-3mm}
\centering
\scriptsize
\renewcommand{\arraystretch}{1.12} 
\begin{tabular}{p{0.20\linewidth}p{0.27\linewidth}p{0.45\linewidth}}
\toprule
Method & Parent selection $f_t$ & Memory update $G_t$ \\
\midrule
Beam search
& $\mathrm{Top}(\mathcal{P}_{t-1}, W_{t-1},M)$.
& Frontier-only: $\Pool_t=\New_t$ and $W_t(z)=r_{\mathrm{PRM}}(z)$ for $z\in\New_t$. \\
\addlinespace[1.5pt]

Standard SMC
& Sample with replacement with $M=N$ from $\mathcal{P}_{t-1}$ with weights $W_{t-1}$.
& Frontier-only: $\Pool_t=\New_t$ with incremental weights, e.g. $W_t(z)=r_{\mathrm{PRM}}(z)/r_{\mathrm{PRM}}(\pa(z))$. \\
\addlinespace[1.5pt]

Greedy Selection
& $\mathrm{Top}(\mathcal{P}_{t-1}, W_{t-1},M)$.
& Persistent union: $\Pool_t=\Pool_{t-1}\uplus\New_t$ and $W_t(z)=r_{\mathrm{PRM}}(z)$. \\
\addlinespace[1.5pt]

Subpool Selection (SPS)
& sample $\Subpool_t \subseteq \Pool_{t-1}$, then $\mathrm{Top}(\Subpool_t, W_{t-1}, M)$.
& Same persistent union and weights as Greedy Selection; only the selection rule changes; see \eqref{eq:sps-subpool-size-method}. \\
\addlinespace[1.5pt]

Power Backtrack SMC
& Sample with replacement with $M=N$ from $\mathcal{P}_{t-1}$ with weights $W_{t-1}$.
& Sample with replacement $Nt$ particles from $(\Pool_{t-1}; W_{t-1})$ and union with $\New_t$; assign mixture-proposal weights by \eqref{eq:pbsmc-weight-main}. \\
\bottomrule
\end{tabular}
\vspace{2mm}
\caption{Instantiation of Algorithm~\ref{alg:master-particle-search}.  Beam search and standard SMC are frontier-only because their next pool contains only $\New_t$.  Greedy Selection, SPS, and Power Backtrack SMC keep historical prefixes selectable, either by multiset union or by resampling historical particles.}
\label{tab:master-instantiations}
\vspace{-3mm}
\end{table} 

\vspace{-3mm}
\section{Algorithms}
\label{sec:algorithms}
\vspace{-1mm}

We now describe our two algorithms in details: Subpool Selection and Power Backtrack SMC. 

\vspace{-1mm}
\subsection{Subpool Selection}
\label{subsec:sps-method}
\vspace{-1mm}
We first introduce persistent Greedy Selection algorithm and explain the blocker problem. Then we present our Subpool Sampling algorithm.
\vspace{-1mm}
\subsubsection{Greedy Selection and the blocker problem}
\label{subsubsec:persistent-TopM-blocker}

Let us consider the simple case where $\mathrm{Top}(\mathcal{P}_{t-1}, W_{t-1},M)$  as selection rule is greedy;
at each round, it simply selects $M$ parents by deterministic top $M$ ranking of the scores, expands each selected parent to generate $B=N/M$ children, and retains both old and new prefixes by updating $\Pool_t=\Pool_{t-1}\uplus\New_t$.  It then assigns $W_t(z)=r_{\mathrm{PRM}}(z)$ to every retained prefix and selects the next parents as $\Cand_{t+1}=\TopPool{\Pool_t}{W_t}{M}$.  The full procedure 
is given in \cref{alg:persistent-TopM} (see the Appendix).

Now suppose that the selected parents at time $t$ are the current global top $M$ prefixes.  Because their scores are at least as large as the scores of all unselected previous particles, the next global top $M$ prefixes will be chosen from only those parents and their newly generated children. 
Hence, greedy selection can never revisit or backtrack to any prefix older than the parents. We name this the greedy selection blocker problem. To deal address this, we next present our proposed algorithm named Subpool Selection (SPS) combining random subpool and top $M$ selection.

\vspace{-1mm}
\subsubsection{Subpool Selection (SPS)}
\label{subsubsec:sps-random-subpool}
\vspace{-1mm}

SPS constructs a persistent pool, but instead selecting top prefixes from the whole pool, it does the selection on a random subpool $\Subpool_t\subseteq \Pool_{t-1}$  of size  
\begin{equation}
    K_t=|\Subpool_t|=\max\{M,\lfloor\rho_t|\Pool_{t-1}|\rfloor\},
    \label{eq:sps-subpool-size-method},
\end{equation}
where $\rho_t$ is the subpooling ratio at time step $t$.  In practice, the PRM takes value between $[0,1]$, we use an adaptive subpooling ratio as the PRM mean $\rho_t=|\Pool_{t-1}|^{-1}\sum_{z\in\Pool_{t-1}}r_{\PRM}(z)$. By selecting the top $M$ candidates within a random subpool, a correct prefix that is poorly ranked in the global pool has a higher chance of being selected. 
This mirrors selective-sampling heuristic search such as SeeA$^*$, where the expanded node is best in a sampled subset rather than necessarily best in a larger set \cite{zhao2024seea}. 
The size of subpool depends on parameter $\rho_t$. A smaller value $\rho_t$ means that the PRM has weak confidence over the whole pool, then we should keep a smaller subpool to avoid the high-score incorrect prefixes and otherwise. 


\vspace{-1mm}
\subsection{Power Backtrack SMC}
\label{subsec:power-backtrack-smc}

We now describe our approach to introduce backtracking to SMC-sampling methods. 

\vspace{-1mm}
\subsubsection{Persistent-pool target}
\label{subsubsec:pbsmc-target}

In Power Backtrack SMC (PB-SMC), the initial pool $\Pool_0$ consists of one-step samples, after round $t$ the pool can contain prefixes with length at most $t+1$; see Algorithm~\ref{alg:power-backtrack-smc} (Appendix).  Let $\Xspace_s$ be the set of prefixes with exactly $s$ reasoning steps. 
%
In step $t$, our algorithm will sample from $\Xspace_{\le t+1}:=\bigcup_{s=1}^{t+1}\Xspace_s$. 
This is the key departure from standard frontier-only SMC, whose prefixes at step $t$ are restricted to the current frontier $\Xspace_{t+1}$.

Recall that $p_{\mathrm{LLM}}(z)$ is the LLM prefix probability and and $r_{\mathrm{PRM}}(z)>0$ is the PRM score of prefix $z$. For any prefix $z$, we use function $\mathrm{len}$ to indicate the length of $z$ . 
For $\beta_t\ge0$ and $z\in \mathcal{X}_{\leq T+1}$,  we define the powered density function at step $t$
\begin{equation}
    \pi_t^{(\beta_t)}(z)
    :=
    \frac{\widetilde\pi_t^{(\beta_t)}(z)}
    {\sum_{y\in\Xspace_{\le T+1}}\widetilde\pi_t^{(\beta_t)}(y)} \quad  \text{where} \quad \widetilde\pi_t^{(\beta_t)}(z)
    :=
    p_{\mathrm{LLM}}(z)r_{\mathrm{PRM}}(z)^{\beta_t}\ind\{\length(z)\le t+1\}.
    \label{eq:powered-historical-target-main}
\end{equation}
Using the power $\beta_t$ on the PRM score will downweight low-score prefixes and accenturate high-score prefixes before normalizing and sampling. 


\vspace{-1mm}
\subsubsection{Expansion and memory}
\label{subsubsec:pbsmc-expansion-memory}

At round $t$, we sample $N$ parents with replacement from $\Pool_{t-1}$ using normalized weights $W_{t-1}$ and expands each selected parent to produce the new-child multiset $\New_t$.  
To keep a pool of size $Nt$ of important historical prefixes, 
 PB-SMC samples a retained multiset $\mathcal{S}_t$ of size $Nt$ with replacement from $\Pool_{t-1}$ with normalized weights $W_{t-1}$. Then we form a joint pool of historical prefixes and newly-generated prefixes  $\Pool_t=\mathcal{S}_t\uplus\New_t$ with  $|\Pool_t|=N(t+1)$. 
At time $t$, we have a mixture proposal,
$q_t^{\mathrm{mix}}=\alpha_t q_t^{\mathrm{new}}+(1-\alpha_t)q_t^{\mathrm{hist}}$,
similar to multiple-importance-sampling~\citep{veach1995optimally}.  The $q_t^{\mathrm{new}}$ is the distribution of prefix from the previous powered distribution and applies one LLM expansion step, while the $q_t^{\mathrm{hist}}$ is distribution of existing prefix from the previous powered distribution without expansion.  The coefficient $\alpha_t$ controls the balance between newly generated children and retained historical particles.  Appendix~\ref{app:pbs-proof} gives
the formulas for 
the weights and the parameters $\alpha_t$ and $\beta_t$. 


\vspace{-1mm}
\subsubsection{Mixture-corrected weights}
\label{subsubsec:pbsmc-mixture-weights}

For prefix $z\in\Pool_t$, let $\pa(z)$ denote its  parent prefix, if $\pa(z)=x_0$, then $\mathrm{len}(z)=1$.  Set $r_{\PRM}(x_0)=1$.  The weight update assigns unnormalized weights inductively
\begin{equation}
    W_t(z)
    :=
    \begin{cases}
    \alpha_t F_t(z), & z\in\New_t,\\[1mm]
    (1-\alpha_t)\dfrac{F_t(z)}{t}, & z\in \mathcal{S}_t,
    \end{cases}
    \label{eq:pbsmc-weight-main}
\end{equation}
where
\begin{equation}
    F_t(z)
    :=
    \frac{
    \left(\dfrac{r_{\mathrm{PRM}}(z)}{r_{\mathrm{PRM}}(\pa(z))}\right)^{\beta_{t-1}}
    r_{\mathrm{PRM}}(z)^{\beta_t-\beta_{t-1}}
    }
    {
    \alpha_t\ind\{\mathrm{len}(z)\ge2\}
    +(1-\alpha_t)
    \left(\dfrac{r_{\mathrm{PRM}}(z)}{r_{\mathrm{PRM}}(\pa(z))}\right)^{\beta_{t-1}}
    \ind\{\mathrm{len}(z)\le t\}
    }.
    \label{eq:power-backtrack-factor-main}
\end{equation}
\begin{theorem}[PB-SMC targets the powered proxy distribution]
\label{thm:pbsmc-consistency-main}
Fix horizon $T$ and a finite prefix space $\Xspace_{\le T+1}$.
Assume $0<R_{\min}\le r_{\mathrm{PRM}}(z)\le R_{\max}<\infty$ for some constants $R_{\min}$ and $R_{\max}$, $0<\alpha_t<1$, and the LLM expansion kernel is normalized.  Let $\widehat\pi_t^N$ be the weighted empirical distribution (produced by PB-SMC) with support $\Pool_{t}$ and associated weight $ W_{t}$.  Then, for every fixed $t\le T$ and bounded function $f:\Xspace_{\le T+1}\to\mathbb{R}$,
\[
    E_{\widehat\pi_t^{(N)}}[f(X)]
    :=
    \frac{\sum_{z\in\Pool_t}W_t(z)f(z)}
    {\sum_{z\in\Pool_t}W_t(z)}
    \xrightarrow[N\to\infty]{p}
    E_{\pi_t^{(\beta_t)}}[f(X)].
\]
If $\beta_t$ is computed by the schedule in Appendix~\ref{app:pbsmc-schedules}, the same statement holds for the adaptive powers.
\end{theorem}

Theorem~\ref{thm:pbsmc-consistency-main} says that the Power Backtrack SMC estimator converges in probability to the powered PRM-defined proxy target at each fixed round, ensuring asymptotic correctness with respect to the proxy objective.

\vspace{-2mm}
\section{Experiments}
\vspace{-1mm}

We evaluate whether persistent-pool backtracking improves the accuracy--compute trade-off in PRM-guided test-time scaling. Section~\ref{sec:setup} summarizes our experimental setup. We then report main results across four mathematical reasoning benchmarks, analyze scaling behavior under varying token and runtime budgets, and study the effects of historical backtracking and stochastic persistent-pool access.



\vspace{-2mm}

\vspace{-1mm}
\subsection{Experimental Setup}
\label{sec:setup}

\vspace{-1mm}
\paragraph{Datasets.}
We evaluate on MATH500~\cite{lightman2024lets}, AMC23~\cite{amc2023}, AIME24~\cite{aime24}, and Minerva Math~\cite{lewkowycz2022solving}, covering intermediate to advanced mathematical reasoning tasks. Full details are in Appendix~\ref{app:exp_setup}.

\vspace{-2mm}
\paragraph{Models and PRMs.}
We conduct experiments with three language models of varying scales: Qwen2.5-7B-Instruct, Qwen2.5-3B-Instruct~\cite{qwen2.5}, and Phi-4-mini-Instruct~\cite{abouelenin2025phi}.
In the main paper, we use Qwen2.5-Math-PRM-7B~\cite{zhang2025lessons} as the default PRM.


\vspace{-2mm}
\paragraph{Baselines and Implementation.} 
We compare with model-default decoding, Self-Consistency~\cite{wang2023selfconsistency}, Best-of-N~\cite{brown2024large}, Beam Search~\cite{snell2025scaling}, SMC~\cite{puri2025probabilistic}, MCTS~\cite{inoue2025widerdeeperscalingllm}, and DVTS~\cite{beeching2024scaling}. 
We additionally report Pass@32~\cite{chen2021evaluating}, computed over 32 independently sampled solutions, as a reference for the performance attainable with a larger sampling budget. 
All sampling methods use temperature $0.7$. 
Baselines use $N=32$ expansions per step, whereas PB-SMC uses $N=8$. For all main experiments with SPS, we set $M=N=8$, allowing us to evaluate whether persistent-pool backtracking can achieve competitive performance under a smaller expansion budget. 
Results are averaged over 5 seeds for AMC23 and AIME24, and 3 seeds for Minerva Math and MATH500.





\vspace{-1mm}
\subsection{Main Results}
\label{sec:results}

Table~\ref{tab:main_search} presents our main experimental results comparing PB-SMC and SPS against competitive baselines across four mathematical reasoning benchmarks.
We report results for two representative LLMs in the main paper; more results including all three LLMs are provided in Appendix~\ref{sec:full_results}.
Across the reported model--benchmark pairs, our proposed methods consistently match or outperform competitive baselines while using substantially fewer tokens.
Notably, both methods achieve competitive accuracy using only $N=8$ candidates compared to $N=32$ candidates for all baselines.
Averaged across all model--benchmark pairs (Appendix~\ref{sec:full_results}), SPS improves accuracy while reducing token consumption by 49.4\%, and PB-SMC improves accuracy with a 36.7\% token reduction.


On Qwen2.5-7B-Instruct, SPS achieves the best overall performance, reaching $54.66\%$ average accuracy and outperforming strong PRM-guided baselines such as DVTS ($51.40\%$) and Beam Search ($52.12\%$) while consuming fewer generated tokens. On Phi-4-mini-Instruct, PB-SMC obtains the highest overall accuracy ($48.69\%$). These results show that the two proposed backtracking methods provide consistent gains across different LLMs, with SPS performing strongest on Qwen2.5-7B-Instruct and PB-SMC performing strongest on Phi-4-mini-Instruct.

\begin{table*}[t]
\vspace{-3mm}
\centering
\scriptsize
\caption{%
    \textbf{Comparison of test time scaling methods across benchmarks and models.} 
    We report accuracy (Acc, $\uparrow$ better) and token usage (\#Tokens, $\downarrow$ better) across four benchmarks.
    Results include mean and standard deviation over 5 random seeds for AMC23 vs AIME24 and 3 random seeds for Minerva vs MATH500.
    \textbf{Bold} indicates best performance; \underline{underline} indicates runner-up.
}
\label{tab:main_search}
\resizebox{\linewidth}{!}{%
\begin{tabular}{l c cc cc cc cc cc}
\toprule
& & \multicolumn{2}{c}{\textbf{Minerva}}
& \multicolumn{2}{c}{\textbf{Math500}}
& \multicolumn{2}{c}{\textbf{AMC23}}
& \multicolumn{2}{c}{\textbf{AIME24}}
& \multicolumn{2}{c}{\textbf{Average}} \\
\cmidrule(lr){3-4}\cmidrule(lr){5-6}\cmidrule(lr){7-8}\cmidrule(lr){9-10}\cmidrule(lr){11-12}
\textbf{Method} & \textbf{N} &
\textbf{Accuracy\,$\uparrow$} & \textbf{\#Tokens\,$\downarrow$} &
\textbf{Accuracy\,$\uparrow$} & \textbf{\#Tokens\,$\downarrow$} &
\textbf{Accuracy\,$\uparrow$} & \textbf{\#Tokens\,$\downarrow$} &
\textbf{Accuracy\,$\uparrow$} & \textbf{\#Tokens\,$\downarrow$} &
\textbf{Accuracy\,$\uparrow$} & \textbf{\#Tokens\,$\downarrow$} \\
\midrule

\multicolumn{12}{c}{\textbf{Qwen2.5-7B-Instruct}} \\
\cmidrule{1-12}

Base Model & 1 & \sd{34.68}{1.12} & 617 & \sd{76.00}{0.80} & 584 & \sd{54.50}{2.92} & 874 & \sd{11.33}{1.63} & 1039 & 44.13 & 779 \\
\addlinespace[2pt]

Pass@\textit{k}$^\dagger$ & 32 & \sd{52.33}{0.21} & 19765 & \sd{93.67}{0.12} & 18788 & \sd{90.50}{2.45} & 27979 & \sd{31.33}{3.40} & 33262 & 66.96 & 24949 \\
\addlinespace[2pt]

Self-Consistency & 32 & \sd{38.85}{0.56} & 19765 & \sd{83.27}{0.31} & 18788 & \sd{62.00}{1.87} & 27979 & \sd{16.67}{0.00} & 33262 & 50.20 & 24949 \\
\addlinespace[2pt]

Best-of-N & 32 & \sd{38.48}{0.42} & 19765 & \sd{86.40}{0.20} & 18788 & \sd{65.00}{1.58} & 27564 & \textbf{\sd{20.67}{1.33}} & 33344 & 52.64 & 24865 \\
\addlinespace[2pt]

Beam Search & 32 & \sd{40.20}{0.56} & 23900 & \sd{86.27}{0.31} & 22771 & \sd{62.00}{2.92} & 39195 & \sd{20.00}{2.11} & 55150 & 52.12 & 35254 \\
\addlinespace[2pt]

DVTS & 32 & \sd{39.22}{0.43} & 23549 & \sd{86.73}{0.23} & 22978 & \sd{63.00}{1.00} & 34670 & \sd{16.67}{2.11} & 46725 & 51.41 & 31981 \\
\addlinespace[2pt]

MCTS & 32 & \sd{36.27}{0.87} & 25208 & \sd{78.47}{0.81} & 22006 & \sd{58.00}{1.00} & 30328 & \sd{16.67}{1.33} & 35644 & 47.35 & 28297 \\
\addlinespace[2pt]

SMC & 32 & \sd{40.32}{0.94} & 21642 & \sd{86.67}{0.41} & 23307 & \sd{62.00}{3.54} & 35546 & \sd{18.65}{2.66} & 41031 & 51.91 & 30382 \\
\addlinespace[2pt]

\cmidrule{1-12}
\addlinespace[1pt]

\textbf{PB-SMC (ours)} & 8 & \underline{\sd{41.79}{0.56}} & \underline{13216} & \underline{\sd{87.47}{0.31}} & \underline{10628} & \underline{\sd{66.50}{2.85}} & \textbf{17637} & \sd{19.33}{1.50} & \underline{27581} & \underline{53.77} & \underline{17266} \\
\addlinespace[2pt]


\textbf{SPS (ours)} & \textbf{8} & \textbf{\sd{43.26}{0.77}} & \textbf{9387} & \textbf{\sd{87.87}{0.70}} & \textbf{9076} & \textbf{\sd{67.50}{2.50}} & \underline{22446} & \underline{\sd{20.00}{2.37}} & \textbf{24347} & \textbf{54.66} & \textbf{16314} \\
\addlinespace[2pt]

\midrule

\multicolumn{12}{c}{\textbf{Phi-4-mini-Instruct}} \\
\cmidrule{1-12}

Base Model & 1 & \sd{23.90}{1.11} & 535 & \sd{62.80}{0.80} & 525 & \sd{42.50}{3.54} & 781 & \sd{7.33}{3.27} & 948 & 34.13 & 697 \\
\addlinespace[2pt]

Pass@\textit{k}$^\dagger$ & 32 & \sd{54.53}{0.21} & 17119 & \sd{91.07}{0.12} & 16774 & \sd{79.00}{3.39} & 25003 & \sd{28.67}{5.81} & 30316 & 63.32 & 22303 \\
\addlinespace[2pt]

Self-Consistency & 32 & \sd{33.70}{0.21} & 17119 & \sd{77.00}{0.40} & 16774 & \sd{53.50}{2.55} & 25003 & \sd{13.33}{0.00} & 30316 & 44.38 & 22303 \\
\addlinespace[2pt]

Best-of-N & 32 & \sd{34.80}{0.42} & 17119 & \sd{80.80}{0.20} & 16774 & \sd{50.50}{4.30} & 25003 & \sd{14.67}{2.67} & 30316 & 45.19 & 22303 \\
\addlinespace[2pt]

Beam Search & 32 & \sd{33.46}{0.37} & 23768 & \sd{81.20}{0.20} & 21397 & \sd{56.50}{4.06} & 37513 & \underline{\sd{16.00}{3.89}} & 39635 & 46.79 & 30578 \\
\addlinespace[2pt]

DVTS & 32 & \sd{33.82}{0.74} & 22140 & \sd{81.53}{0.12} & 20956 & \sd{55.00}{2.24} & 36815 & \underline{\sd{16.00}{3.27}} & 37532 & 46.59 & 29361 \\
\addlinespace[2pt]

MCTS & 32 & \sd{32.75}{2.08} & 23108 & \sd{73.20}{0.56} & 28036 & \sd{55.00}{1.50} & 33320 & \sd{14.67}{1.66} & 30316 & 43.91 & 28695 \\
\addlinespace[2pt]

SMC & 32 & \sd{34.44}{0.75} & 24907 & \sd{81.40}{0.59} & 25459 & \sd{54.50}{3.32} & 42574 & \sd{14.00}{0.98} & 45448 & 46.09 & 34597 \\
\addlinespace[2pt]

\cmidrule{1-12}
\addlinespace[1pt]

\textbf{PB-SMC (ours)} & 8 & \textbf{\sd{37.87}{0.37}} & \underline{13198} & \underline{\sd{82.07}{0.50}} & \underline{10422} & \textbf{\sd{57.50}{3.06}} & \underline{15547} & \textbf{\sd{17.33}{3.89}} & \underline{21258} & \textbf{48.69} & \underline{15106} \\
\addlinespace[2pt]


\textbf{SPS (ours)} & \textbf{8} & \underline{\sd{35.42}{0.56}} & \textbf{9832} & \textbf{\sd{82.67}{0.50}} & \textbf{7959} & \underline{\sd{57.00}{2.92}} & \textbf{13497} & \underline{\sd{16.00}{1.33}} & \textbf{15056} & \underline{47.77} & \textbf{11586} \\
\addlinespace[2pt]


\bottomrule
\end{tabular}%
} 
\end{table*}


\vspace{-2mm}
\subsection{Accuracy vs Compute Budget}
\label{sec:acc_budget}
\vspace{-1mm}

To understand how our methods scale with token budget, we vary the number of parallel trajectories $N \in \{4, 8, 16, 24, 32, 48, 64\}$ and plot accuracy versus average number of generated tokens on MATH500 using Qwen2.5-7B-Instruct (Fig. \ref{fig:acc_vs_tokens_across_N}). This visualization reveals the accuracy-token Pareto front: better methods lie higher and further left, achieving higher accuracy with fewer tokens.

\begin{figure}[t]
    \centering

    \begin{minipage}[t]{0.48\linewidth}
        \centering
        \includegraphics[width=\linewidth]{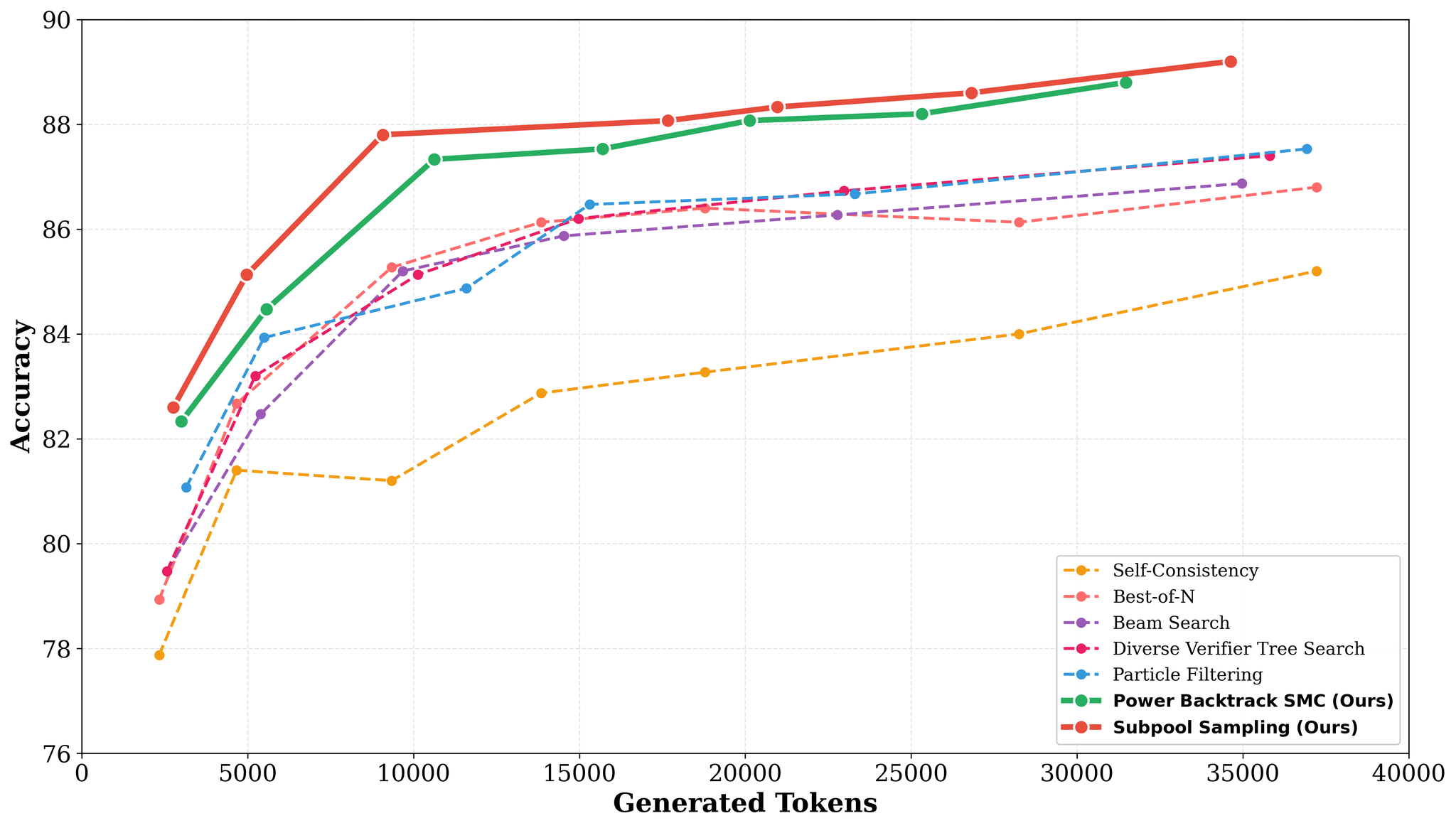}\\[-1mm]
        {\small (a) Accuracy vs token budget.}
    \end{minipage}
    \hfill
    \begin{minipage}[t]{0.48\linewidth}
        \centering
        \includegraphics[width=\linewidth]{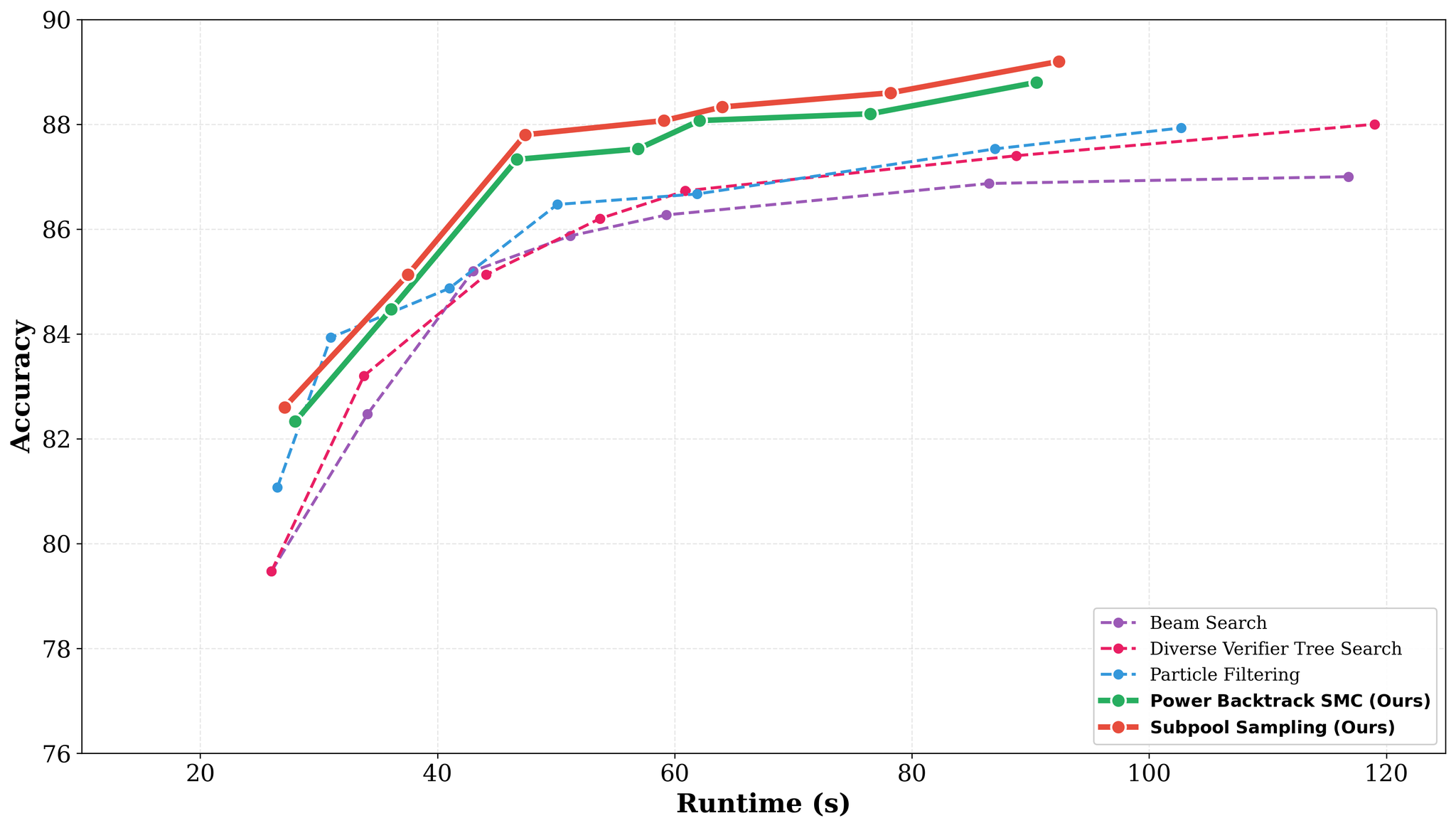}\\[-1mm]
        {\small (b) Accuracy vs runtime.}
    \end{minipage}

    \caption{
    \textbf{Accuracy--compute trade-off on MATH500 using Qwen2.5-7B-Instruct.}
    Each curve varies the number of parallel trajectories from $N=4$ to $N=64$.
    Panel (a) reports accuracy versus generated tokens, while panel (b) reports accuracy versus wall-clock runtime. PB-SMC and SPS improve the Pareto front under both compute measures, achieving higher accuracy at a fixed budget or requiring less compute for a target accuracy.
    }
    \label{fig:acc_vs_tokens_across_N}
    \vspace{-4mm}
\end{figure}

\vspace{-2mm}

\paragraph{Backtracking methods improve the accuracy--token trade-off on MATH500.}
Figure~\ref{fig:acc_vs_tokens_across_N}(a) shows that PB-SMC and SPS consistently achieve significantly more favorable accuracy--token trade-off than the baselines on MATH500. Across the evaluated compute range, these methods often obtain higher accuracy under comparable token budgets, or reach similar accuracy with fewer generated tokens. The gains are most visible in the moderate-compute regime (10K--25K tokens), where improving token efficiency can be particularly useful for practical deployment.


\vspace{-2mm}

\paragraph{Backtracking enables efficient compute allocation.}
Figure~\ref{fig:acc_vs_tokens_across_N}(b) shows that the benefits of PB-SMC and SPS persist under wall-clock runtime, not only under generated-token counts. As discussed in Section~4.1, backtracking may require recomputing KV caches for historical prefixes, introducing additional overhead relative to frontier-only methods; we analyze this trade-off in more detail in Appendix~\ref{app:runtime}. Nevertheless, both methods remain competitive in runtime and the gain becomes more significant in the higher time/accuracy region, indicating that their improved selection of prefixes outweighs the cost of occasional recomputation. Thus, persistent-pool backtracking improves practical compute allocation with an acceptable compute overhead.


\subsection{Effect of Backtracking}
\label{sec:component_ablation}
\vspace{-1mm}

Table~\ref{tab:ablation} compares frontier-only variants with their historical-pool counterparts across four mathematical reasoning benchmarks.  We organize the comparison into three pairs: SMC versus Backtrack SMC, Power SMC versus Power Backtrack SMC, and Greedy Selection versus SPS.  Here, Greedy Selection is the deterministic historical-pool ranking rule in \Cref{subsubsec:persistent-TopM-blocker}: it keeps the persistent pool but selects parents by global ranking. Power SMC is the frontier-only counterpart of Power Backtrack SMC in \Cref{subsec:power-backtrack-smc}: it uses the same powered PRM weighting but sets $\Pool_t=\New_t$, without retaining historical prefixes. Thus the first two pairs isolate adding backtracking to SMC-style sampling, while the last pair isolates replacing global with stochastic subpool selection.

\begin{table*}[t]
\vspace{-1mm}
\centering
\scriptsize
\caption{%
    \textbf{Ablation of historical backtracking and stochastic persistent-pool access with Qwen2.5-7B-Instruct.}
    We report accuracy (Acc, \%, $\uparrow$ better), and average token usage (\#Tokens, $\downarrow$ better).
    \textbf{Bold} indicates best accuracy; \underline{underline} indicates runner-up.
}
\label{tab:ablation}
\resizebox{\linewidth}{!}{%
\begin{tabular}{l c cc cc cc cc cc}
\toprule
& & \multicolumn{2}{c}{\textbf{Minerva}}
& \multicolumn{2}{c}{\textbf{MATH500}}
& \multicolumn{2}{c}{\textbf{AMC23}}
& \multicolumn{2}{c}{\textbf{AIME24}}
& \multicolumn{2}{c}{\textbf{Average}} \\
\cmidrule(lr){3-4}\cmidrule(lr){5-6}\cmidrule(lr){7-8}\cmidrule(lr){9-10}\cmidrule(lr){11-12}
\textbf{Method} & \textbf{N} &
\textbf{Acc\,$\uparrow$} & \textbf{\#Tokens\,$\downarrow$} &
\textbf{Acc\,$\uparrow$} & \textbf{\#Tokens\,$\downarrow$} &
\textbf{Acc\,$\uparrow$} & \textbf{\#Tokens\,$\downarrow$} &
\textbf{Acc\,$\uparrow$} & \textbf{\#Tokens\,$\downarrow$} &
\textbf{Acc\,$\uparrow$} & \textbf{\#Tokens\,$\downarrow$} \\
\midrule

SMC
& 16
& \sd{38.73}{0.56} & 10841
& \sd{84.20}{0.40} & 10228
& \sd{59.50}{5.70} & 14740
& \sd{14.00}{3.67} & 18725
& 49.11 & 13634 \\
\addlinespace[2pt]

Backtrack SMC
& 8
& \underline{\sd{42.89}{0.76}} & 13765
& \sd{85.00}{0.60} & 13332
& \sd{62.50}{5.59} & 19238
& \sd{18.67}{1.83} & 24228
& 52.27 & 17641 \\
\addlinespace[3pt]
\cmidrule{1-12}
\addlinespace[1pt]

Power SMC
& 16
& \sd{39.22}{0.21} & 13274
& \sd{86.00}{0.20} & 11911
& \sd{63.00}{4.11} & 19260
& \sd{16.00}{2.80} & 24988
& 51.06 & 17358 \\
\addlinespace[2pt]

Power Backtrack SMC
& 8
& \sd{41.79}{0.56} & 13216
& \underline{\sd{87.47}{0.31}} & 10628
& \underline{\sd{66.50}{2.85}} & 17637
& \underline{\sd{19.33}{1.50}} & 27581
& \underline{53.77} & 17266 \\
\addlinespace[3pt]
\cmidrule{1-12}
\addlinespace[1pt]

Greedy Selection
& 8
& \sd{40.07}{0.37} & 9972
& \sd{87.20}{0.20} & 9031
& \sd{62.50}{4.68} & 20953
& \sd{16.67}{4.07} & 34033
& 51.61 & 18497 \\
\addlinespace[2pt]

\textbf{SPS}
& \textbf{8}
& \textbf{\sd{43.26}{0.77}} & 9387
& \textbf{\sd{87.87}{0.70}} & 9076
& \textbf{\sd{67.50}{2.50}} & 22446
& \textbf{\sd{20.00}{2.37}} & 24347
& \textbf{54.66} & 16314 \\
\addlinespace[2pt]

\bottomrule
\end{tabular}%
} 
\vspace{-3mm}
\end{table*}


Backtracking consistently improves SMC-style search. Backtrack SMC outperforms standard SMC on all four benchmarks, improving average accuracy from $49.11\%$ to $52.27\%$. The same pattern holds under powered weighting: Power Backtrack SMC improves over Power SMC on every benchmark and increases average accuracy from $51.06\%$ to $53.77\%$. These results indicate that retaining historical prefixes provides consistent gains over frontier-only resampling. Powering the PRM weights further improves SMC-style methods, suggesting that sharper reward-weighted resampling is complementary to historical backtracking.

Backtracking via Subpooling in SPS improves over Greedy Selection on all four benchmarks, increasing average accuracy from $51.61\%$ to $54.66\%$. It also achieves the highest accuracy on every benchmark, with especially large gains over standard SMC on AMC23 and AIME24. Overall, these results show that both historical backtracking and stochastic persistent-pool access contribute to the strong performance of SPS.

\vspace{-1mm}
\section{Conclusion}
\vspace{-1mm}
We identified irreversible frontier-only selection as a potential failure mode in PRM-guided test-time scaling: noisy intermediate scores can permanently discard prefixes that still admit correct continuations. We then introduced persistent-pool stochastic backtracking, a linear-memory alternative that keeps historical prefixes eligible for future expansion without storing the full search tree. We instantiated this framework with Subpool Selection for greedy PRM-guided search and Power Backtrack SMC 
with mixture-corrected historical resampling
for sampling-based search. Across mathematical reasoning benchmarks, our methods improve the accuracy--token trade-off over strong frontier-only baselines. One limitation is that our methods can incur additional runtime overhead when backtracking requires recomputing KV caches for historical prefixes, however empirically we show that this overhead is small comparing to the performance gain elsewhere. Our results suggest that effective test-time scaling should optimize not only how prefixes are scored, but also which prefixes remain available and should be considered for future compute.

\newpage
\bibliographystyle{abbrv}
\bibliography{main}


\appendix
\newpage

\section{Algorithms for Persistent-Pool Search}
\label{app:persistent-pool-algorithms}

The main text keeps only the master template; the method-specific pseudocode is collected here. 
All three algorithms use the notation of Algorithm~\ref{alg:master-particle-search}.  The child budget is $N$, the parent budget is $M$, and each selected parent generates $B=N/M$ children.

\begin{algorithm}[H]
\caption{Greedy Selection}
\label{alg:persistent-TopM}
\begin{algorithmic}[1]
\Require prompt $x_0$, child budget $N$, parent budget $M$, horizon $T$, PRM score $R$.
\State Initialize $\Pool_0$ with $N$ one-step samples from $p_{\mathrm{LLM}}(\cdot\mid x_0)$ and set $W_0(z)\gets r_{\mathrm{PRM}}(z)$ for all $z\in\Pool_0$.
\For{$t=1,\ldots,T$}
    \State $\Cand_t\gets\mathrm{Top}(\mathcal{P}_{t-1}, W_{t-1},M).$
    \State Generate $\New_t\gets\bigcup_{c\in\Cand_t}\textsc{LLM-Expand}(c,N/M)$ and score each $z\in\New_t$.
    \State $\Pool_t\gets\Pool_{t-1}\uplus\New_t$.
    \State Set $W_t(z)\gets r_{\mathrm{PRM}}(z)$ for all $z\in\Pool_t$.
\EndFor
\State \Return final candidates from $\TopPool{\Pool_T}{W_T}{M}$.
\end{algorithmic}
\end{algorithm}

\begin{algorithm}[H]
\caption{Subpool Selection}
\label{alg:sps}
\begin{algorithmic}[1]
\Require prompt $x_0$, child budget $N$, parent budget $M$, horizon $T$, PRM score $R$, subpool ratios $\rho_t\in(0,1]$.
\State Initialize $\Pool_0$ with $N$ one-step samples from $p_{\mathrm{LLM}}(\cdot\mid x_0)$ and set $W_0(z)\gets r_{\mathrm{PRM}}(z)$ for all $z\in\Pool_0$.
\For{$t=1,\ldots,T$}
    \State $K_t\gets\max\{M,\lfloor\rho_t|\Pool_{t-1}|\rfloor\}$.
    \State Sample $\Subpool_t\subseteq\Pool_{t-1}$ uniformly without replacement with $|\Subpool_t|=K_t$.
    \State $\Cand_t\gets\mathrm{Top}(\Subpool_{t}, W_{t-1},M)$.
    \State Generate $\New_t\gets\bigcup_{c\in\Cand_t}\textsc{LLM-Expand}(c,N/M)$ and score each $z\in\New_t$.
    \State $\Pool_t\gets\Pool_{t-1}\uplus\New_t$.
    \State Set $W_t(z)\gets r_{\mathrm{PRM}}(z)$ for all $z\in\Pool_t$.
\EndFor
\State \Return final candidates from $\TopPool{\Pool_T}{W_T}{M}$.
\end{algorithmic}
\end{algorithm}

\begin{algorithm}[H]
\caption{Power Backtrack SMC}
\label{alg:power-backtrack-smc}
\begin{algorithmic}[1]
\Require prompt $x_0$, child budget $N$, horizon $T$, PRM score $R$, powers $\beta_t$, mixture probabilities $\alpha_t$.
\State Initialize $\Pool_0$ with $N$ one-step samples from $p_{\mathrm{LLM}}(\cdot\mid x_0)$ and set $W_0(z)\gets r_{\mathrm{PRM}}(z)$ for all $z\in\Pool_0$.
\For{$t=1,\ldots,T$}
    \State Sample $\Cand_t$ as $N$ parents with replacement from $\Pool_{t-1}$ using normalized weights induced by $W_{t-1}$.
    \State Generate $\New_t\gets\bigcup_{c\in\Cand_t}\textsc{LLM-Expand}(c,1)$ and score each $z\in\New_t$.
    \State Sample a retained multiset $S_t$ of size $Nt$ with replacement from $\Pool_{t-1}$ using normalized weights induced by $W_{t-1}$.
    \State Set $\Pool_t\gets S_t\uplus\New_t$.
    \For{each $z\in\Pool_t$}
        \State Compute $F_t(z)$ using \eqref{eq:power-backtrack-factor-main}.
        \If{$z\in\New_t$}
            \State $W_t(z)\gets\alpha_tF_t(z)$.
        \Else
            \State $W_t(z)\gets(1-\alpha_t)\dfrac{F_t(z)}{t}$.
        \EndIf
    \EndFor
\EndFor
\State \Return final candidates sampled from $\Pool_T$ using normalized $W_T$.
\end{algorithmic}
\end{algorithm}

\section{Mixture-proposal multiple importance sampling and Power Backtrack SMC}
\label{app:pbs-proof}

In this part, we will derive the PB-SMC weighting rule and proves proxy-target consistency. \newline
All distributions are defined on the fixed state space $\Xspace_{\le T+1}:=\bigcup_{s=1}^{T+1}\Xspace_s$.  Since PB-SMC initializes $\Pool_0$ with one-step prefixes, the round-$t$ target has effective support $\Xspace_{\le t+1}$ through the indicator $\ind\{\length(z)\le t+1\}$.

\subsection{A mixture-proposal multiple importance sampling identity}
\label{app:mis-identity}

\begin{lemma}[Mixture-proposal multiple importance sampling identity]
\label{lem:mis-identity}
Let $P$ be a normalized target distribution on a common measurable state space.  Let $q_1,\ldots,q_m$ be normalized proposal distributions, let $\alpha_i\ge0$ with $\sum_{i=1}^m\alpha_i=1$, and define $q_{\mathrm{mix}}(x):=\sum_{i=1}^m\alpha_iq_i(x)$.  If $P\ll q_{\mathrm{mix}}$, then for every bounded measurable $f$,
\begin{equation}
    \bbE_P[f(X)]
    =
    \sum_{i=1}^m
    \bbE_{q_i}\left[
        f(X)\frac{\alpha_iP(X)}{q_{\mathrm{mix}}(X)}
    \right].
    \label{eq:mis-normalized-app}
\end{equation}
\end{lemma}

\begin{proof}
Define $b_i(x):=\alpha_iq_i(x)/q_{\mathrm{mix}}(x)$ whenever $q_{\mathrm{mix}}(x)>0$.  Then $\sum_i b_i(x)=1$, and
\begin{align}
    \bbE_P[f(X)]
    &=\int f(x)P(x)\sum_{i=1}^m b_i(x)\dd x \\
    &=\sum_{i=1}^m\int f(x)P(x)\frac{\alpha_iq_i(x)}{q_{\mathrm{mix}}(x)}\dd x \\
    &=\sum_{i=1}^m
    \bbE_{q_i}\left[
        f(X)\frac{\alpha_iP(X)}{q_{\mathrm{mix}}(X)}
    \right].
\end{align}
\end{proof}

\begin{corollary}[Self-normalized unnormalized form]
\label{cor:mis-self-normalized}
Suppose $P(x)=\widetilde P(x)/Z_P$ and $q_{\mathrm{mix}}(x)=\widetilde q_{\mathrm{mix}}(x)/Z_q$, where $Z_P$ and $Z_q$ may be unknown.  Then
\begin{equation}
\bbE_P[f(X)]
=
\frac{
\sum_{i=1}^m
\bbE_{q_i}\!\left[
  f(X)\alpha_i
  \frac{\widetilde P(X)}{\widetilde q_{\mathrm{mix}}(X)}
\right]
}{
\sum_{i=1}^m
\bbE_{q_i}\!\left[
  \alpha_i
  \frac{\widetilde P(X)}{\widetilde q_{\mathrm{mix}}(X)}
\right]
}.
\label{eq:mis-self-normalized-app}
\end{equation}
\end{corollary}

\begin{proof}
Substitute $P=\widetilde P/Z_P$ and $q_{\mathrm{mix}}=\widetilde q_{\mathrm{mix}}/Z_q$ into Lemma~\ref{lem:mis-identity}.  Both numerator and denominator contain the same multiplicative factor $Z_q/Z_P$, which cancels in the self-normalized ratio.
\end{proof}

\subsection{Shared proposal normalizer for PB-SMC}
\label{app:pbsmc-shared-normalizer}

In this part, we will prove that  the generated-child and retained-history proposals have the same normalizer. \newline
At round $t\ge1$, the previous powered target is
\[
    \widetilde\pi_{t-1}^{(\beta_{t-1})}(z)
    =
    p_{\mathrm{LLM}}(z)r_{\mathrm{PRM}}(z)^{\beta_{t-1}}\ind\{\length(z)\le t\},
\]
with normalizing constant
\[
    Z_{t-1}
    :=
    \sum_{z\in\Xspace_{\le T+1}}
    \widetilde\pi_{t-1}^{(\beta_{t-1})}(z)
    =
    \sum_{u\in\Xspace_{\le t}}p_{\mathrm{LLM}}(u)r_{\mathrm{PRM}}(u)^{\beta_{t-1}}.
\]
The retained-history proposal redraws a prefix from the previous powered target, so its unnormalized density is
\begin{equation}
    \widetilde q_t^{\mathrm{hist}}(z)
    =
    p_{\mathrm{LLM}}(z)r_{\mathrm{PRM}}(z)^{\beta_{t-1}}\ind\{\length(z)\le t\},
    \qquad
    \sum_z \widetilde q_t^{\mathrm{hist}}(z)=Z_{t-1}.
    \label{eq:qhist-unnormalized-app}
\end{equation}

The one-step prefixes are generated during initialization.  Hence, for rounds $t\ge1$, the retained-history prefixes have length of $1,\ldots,t$, thus the generated-child prefixes have length of $2,\ldots,t+1$.  For the generated-child proposal, write $z=x_{1:d}$ with $2\le d\le t+1$ and parent $u=x_{1:d-1}$.  Sampling $u$ from the previous powered target and then sampling $x_d$ from the LLM gives
\begin{align}
    \widetilde q_t^{\mathrm{new}}(x_{1:d})
    &=
    p_{\mathrm{LLM}}(x_{1:d-1})r_{\mathrm{PRM}}(x_{1:d-1})^{\beta_{t-1}}
    p_{\mathrm{LLM}}(x_d\mid x_{1:d-1})
    \\
    &=
    p_{\mathrm{LLM}}(x_{1:d})r_{\mathrm{PRM}}(x_{1:d-1})^{\beta_{t-1}}.
    \label{eq:qnew-prefix-app}
\end{align}
Equivalently, with $\pa(z)$ denoting the structural parent of $z$,
\begin{equation}
    \widetilde q_t^{\mathrm{new}}(z)
    =
    p_{\mathrm{LLM}}(z)r_{\mathrm{PRM}}(\pa(z))^{\beta_{t-1}}\ind\{2\le\length(z)\le t+1\}.
    \label{eq:qnew-unnormalized-app}
\end{equation}
Because the LLM extension kernel is normalized,
\begin{align}
\sum_{z\in\Xspace_{\le T+1}}\widetilde q_t^{\mathrm{new}}(z)
&=
\sum_{d=2}^{t+1}
\sum_{x_{1:d-1}\in\Xspace_{d-1}}
p_{\mathrm{LLM}}(x_{1:d-1})r_{\mathrm{PRM}}(x_{1:d-1})^{\beta_{t-1}}
\sum_{x_d}p_{\mathrm{LLM}}(x_d\mid x_{1:d-1}) \\
&=
\sum_{d=2}^{t+1}
\sum_{x_{1:d-1}\in\Xspace_{d-1}}
p_{\mathrm{LLM}}(x_{1:d-1})r_{\mathrm{PRM}}(x_{1:d-1})^{\beta_{t-1}} \\
&=
\sum_{u\in\Xspace_{\le t}}p_{\mathrm{LLM}}(u)r_{\mathrm{PRM}}(u)^{\beta_{t-1}}
=
Z_{t-1}.
\end{align}
Thus $\widetilde q_t^{\mathrm{new}}$ and $\widetilde q_t^{\mathrm{hist}}$ share the same normalizing constant.  Consequently,
\begin{equation}
    \widetilde q_t^{\mathrm{mix}}(z)
    =
    \alpha_t\widetilde q_t^{\mathrm{new}}(z)
    +(1-\alpha_t)\widetilde q_t^{\mathrm{hist}}(z)
    \label{eq:pbsmc-unnormalized-mixture-app}
\end{equation}
also has normalizer $Z_{t-1}$.

\subsection{Explicit Power Backtrack SMC weight}
\label{app:pbs-explicit-weight}
In this part, we will derive the weights for Power Backtrack SMC method.
\begin{lemma}[Explicit PB-SMC correction factor]
\label{lem:pbs-explicit-weight}
Let $z=x_{1:d}\in\Xspace_{\le t+1}$ with $d\ge1$.  Let $\pa(z)=x_{1:d-1}$ for $d\ge2$, let $\pa(z)=x_0$ for $d=1$, and set $R(x_0)=1$.  Under the mixture proposal in \eqref{eq:pbsmc-unnormalized-mixture-app},
\[
    F_t(z)
    :=
    \frac{\widetilde\pi_t^{(\beta_t)}(z)}
    {\widetilde q_t^{\mathrm{mix}}(z)}
\]
equals
\begin{equation}
    F_t(z)
    =
    \frac{
    \left(\frac{r_{\mathrm{PRM}}(z)}{r_{\mathrm{PRM}}(\pa(z))}\right)^{\beta_{t-1}}
    r_{\mathrm{PRM}}(z)^{\beta_t-\beta_{t-1}}
    }
    {
    \alpha_t\ind\{d\ge2\}
    +(1-\alpha_t)
    \left(\frac{r_{\mathrm{PRM}}(z)}{r_{\mathrm{PRM}}(\pa(z))}\right)^{\beta_{t-1}}
    \ind\{d\le t\}
    }.
    \label{eq:power-backtrack-factor-app}
\end{equation}
\end{lemma}

\begin{proof}
The target numerator is
\[
    \widetilde\pi_t^{(\beta_t)}(z)
    =
    p_{\mathrm{LLM}}(z)r_{\mathrm{PRM}}(z)^{\beta_t}\ind\{d\le t+1\}.
\]
By \eqref{eq:qnew-unnormalized-app} and \eqref{eq:qhist-unnormalized-app},
\[
    \widetilde q_t^{\mathrm{new}}(z)
    =
    p_{\mathrm{LLM}}(z)r_{\mathrm{PRM}}(\pa(z))^{\beta_{t-1}}\ind\{d\ge2\},
    \qquad
    \widetilde q_t^{\mathrm{hist}}(z)
    =
    p_{\mathrm{LLM}}(z)r_{\mathrm{PRM}}(z)^{\beta_{t-1}}\ind\{d\le t\}.
\]
Therefore
\begin{align}
    \widetilde q_t^{\mathrm{mix}}(z)
    &=
    p_{\mathrm{LLM}}(z)r_{\mathrm{PRM}}(\pa(z))^{\beta_{t-1}}
    \left[
    \alpha_t\ind\{d\ge2\}
    +(1-\alpha_t)
    \left(\frac{r_{\mathrm{PRM}}(z)}{r_{\mathrm{PRM}}(\pa(z))}\right)^{\beta_{t-1}}
    \ind\{d\le t\}
    \right].
\end{align}
Dividing $\widetilde\pi_t^{\beta_t}(z)$ by this expression and using $d\le t+1$ for the round-$t$ target gives \eqref{eq:power-backtrack-factor-app}.  This is the same expression as \eqref{eq:power-backtrack-factor-main}.
\end{proof}

\begin{lemma}[Fixed-count mixture-proposal estimator]
\label{lem:fixed-count-estimator}
Let $\New_t$ contain $N$ samples from $q_t^{\mathrm{new}}$ and let $S_t$ contain $Nt$ samples from $q_t^{\mathrm{hist}}$.  The self-normalized mixture-proposal estimator of $\pi_t^{\beta_t}(f)$ uses unnormalized weights
\begin{equation}
    W_t(z)=
    \begin{cases}
    \alpha_tF_t(z), & z\in\New_t,\\[1mm]
    (1-\alpha_t)\dfrac{F_t(z)}{t}, & z\in S_t.
    \end{cases}
    \label{eq:fixed-count-weights-app}
\end{equation}
\end{lemma}

\begin{proof}
Corollary~\ref{cor:mis-self-normalized} gives the numerator
\[
    \frac1N\sum_{z\in\New_t} f(z)\alpha_tF_t(z)
    +
    \frac1{Nt}\sum_{z\in S_t} f(z)(1-\alpha_t)F_t(z),
\]
with the analogous denominator for $f\equiv1$.  Multiplying numerator and denominator by $N$ gives \eqref{eq:fixed-count-weights-app}.  Multiplying all weights by a common positive constant does not change the self-normalized estimator.
\end{proof}

\subsection{Adaptive schedules for powers and mixture probabilities}
\label{app:pbsmc-schedules}

The power schedule controls how aggressively PB-SMC sharpens the powered PRM
target during search. Let $\Pool_t=S_t\uplus\New_t$ be the weighted pool
at round $t$, and let $C_t=|\Pool_t|$; with the default retained-history size
$H_t=Nt$, we have $C_t=N(t+1)$. We normalize the PRM scores on this multiset as
$a_z=r_{\mathrm{PRM}}(z)/\sum_{y\in\Pool_t}r_{\mathrm{PRM}}(y)$ and define the concentration statistic
$\sigma_t=\sum_{z\in\Pool_t}a_z^2$, so that $1/C_t\le\sigma_t\le1$. For a
power-step hyperparameter $\gamma>0$, we set
\begin{equation}
    \beta_{t+1}=\beta_t+\Delta_t,
    \qquad
    \Delta_t=\gamma\left(1-\left(\sigma_t-\frac1{C_t}\right)\right).
    \label{eq:pbsmc-beta-schedule-app}
\end{equation}
This gives $\gamma/C_t\le\Delta_t\le\gamma$, so $\beta_t$ is strictly increasing.
When the normalized PRM scores are nearly uniform, $\sigma_t$ is close to $1/C_t$,
and the schedule increases $\beta_t$ quickly to move toward stronger exploitation.
When the scores are already concentrated on a few particles, $\sigma_t$ is close
to $1$, and the schedule makes a smaller update, slowing further collapse. Thus,
$\gamma$ controls the overall rate at which PB-SMC moves from exploration
toward PRM-driven exploitation. In Section~\ref{sec:ablation_pbsmc_hyperparams},
we evaluate the sensitivity of PB-SMC to $\gamma$ on a held-out MATH training
validation set and use the selected default $\gamma=9$ in all main experiments.

The mixture schedule controls the relative contribution of newly generated children
and retained historical particles. We parameterize it through the history-to-new
ratio
\[
    g_t=\frac{1-\alpha_t}{\alpha_t},
    \qquad
    \alpha_t=\frac{1}{1+g_t}.
\]
We choose $g_t$ by linearly decreasing from $g_{\max}$ to $g_{\min}$:
\begin{equation}
    g_t
    =
    g_{\max}
    -
    \frac{t-1}{T-1}(g_{\max}-g_{\min}),
    \qquad
    \alpha_t=\frac1{1+g_t}.
    \label{eq:pbsmc-alpha-schedule-app}
\end{equation}
For $T=1$, we set $g_t=g_{\min}$. In all experiments, we use $g_{\max}=1$, so
$\alpha_t\ge1/2$ and the mixture weight on newly generated children is never
smaller than the mixture weight on retained historical particles. The terminal
value $g_{\min}$ determines the late-stage amount of stochastic backtracking:
smaller $g_{\min}$ makes $\alpha_t$ closer to one and emphasizes newly generated
children, making PB-SMC closer to frontier-only SMC; larger $g_{\min}$ allocates
more mixture mass to retained historical particles, increasing backtracking but
potentially revisiting stale prefixes more often. In
Section~\ref{sec:ablation_pbsmc_hyperparams}, we evaluate the sensitivity of
PB-SMC to $g_{\min}$ on the same held-out MATH training validation set and use the
selected default $g_{\min}=0.4$ in all main experiments.

\subsection{Proxy-target consistency}
\label{app:pbs-consistency-proof}

PB-SMC consistency follows from self-normalized importance sampling (SNIS).
We use the standard self-normalized importance sampling consistency theorem: if $X_1,\ldots,X_n\overset{\mathrm{i.i.d.}}{\sim}q$, $p\ll q$, and $f$ is integrable under $p$, then
\[
    \frac{\sum_{i=1}^{n} f(X_i)p(X_i)/q(X_i)}
    {\sum_{i=1}^{n} p(X_i)/q(X_i)}
    \xrightarrow[n\to\infty]{a.s.}
    \int f(x)p(x)\dd x
\]
\citep[Theorem~9.2]{mcbook}.  We apply this result on the finite state space $\Xspace_{\le T+1}$.

\begin{theorem}[PB-SMC proxy-target consistency]
\label{thm:pbs-consistency-app}
Fix $T<\infty$ and assume
\[
    |\Xspace_{\le T+1}|<\infty,\qquad
    0<R_{\min}\le r_{\mathrm{PRM}}(z)\le R_{\max}<\infty,\qquad
    0<\alpha_t<1,
\]
and $\sum_{x'}p_{\mathrm{LLM}}(x'\mid z)=1$ for every prefix $z$.  Let PB-SMC initialize $\Pool_0$ with $N$ one-step samples and weights $W_0(z)=r_{\mathrm{PRM}}(z)^{\beta_0}$.  For every fixed $t\le T$ and bounded $f:\Xspace_{\le T+1}\to\mathbb{R}$,
\begin{equation}
    \bbE_{\widehat\pi_t^N}[f(X)]
    :=
    \frac{\sum_{z\in\Pool_t}W_t(z)f(z)}
    {\sum_{z\in\Pool_t}W_t(z)}
    \xrightarrow[N\to\infty]{p}
    \bbE_{\pi_t^{\beta_t}}[f(X)].
    \label{eq:pbsmc-consistency-app}
\end{equation}
If $\beta_t$ is computed by \eqref{eq:pbsmc-beta-schedule-app}, then the same convergence holds with the adaptive powers.
\end{theorem}

\begin{proof}
Since $\Xspace_{\le T+1}$ is finite and $R$ is bounded above and below, all weights and normalizing constants are finite and positive.

\textbf{Base case.}
At $t=0$, the proposal is
\[
    q_0(z)=p_{\mathrm{LLM}}(z)\ind\{\length(z)=1\},
\]
and the target is
\[
    \widetilde\pi_0^{(\beta_0)}(z)
    =
    p_{\mathrm{LLM}}(z)r_{\mathrm{PRM}}(z)^{\beta_0}\ind\{\length(z)\le1\}.
\]
Thus the initial importance ratio is proportional to $r_{\mathrm{PRM}}(z)^{\beta_0}$ on $\Xspace_1$.  By self-normalized importance sampling consistency \citep[Theorem~9.2]{mcbook},
\[
    \bbE_{\widehat\pi_0^N}[f(X)]
    \xrightarrow[N\to\infty]{p}
    \bbE_{\pi_0^{(\beta_0)}}[f(X)].
\]

\textbf{Induction step.}
Assume the conclusion is correct until step $t-1$, that means $\bbE_{\widehat\pi_{t-1}^N}[f(X)]\to\bbE_{\pi_{t-1}^{(\beta_{t-1})}}[f(X)]$ in probability for every bounded $f$.  Multinomial sampling from $\Pool_{t-1}$ using normalized $W_{t-1}$ therefore gives empirical samples converging to $\pi_{t-1}^{(\beta_{t-1})}$.  Applying the normalized LLM kernel gives generated-child samples converging to
\[
    q_t^{\mathrm{new}}(x_{1:d})
    =
    \frac{
    p_{\mathrm{LLM}}(x_{1:d})r_{\mathrm{PRM}}(x_{1:d-1})^{\beta_{t-1}}\ind\{2\le d\le t+1\}
    }{Z_{t-1}},
\]
and redrawing historical prefixes gives samples converging to
\[
    q_t^{\mathrm{hist}}(z)
    =
    \frac{
    p_{\mathrm{LLM}}(z)r_{\mathrm{PRM}}(z)^{\beta_{t-1}}\ind\{\length(z)\le t\}
    }{Z_{t-1}}.
\]
By Section~\ref{app:pbsmc-shared-normalizer}, both proposals share normalizer $Z_{t-1}$, and their mixture is
\[
    q_t^{\mathrm{mix}}
    =
    \alpha_t q_t^{\mathrm{new}}
    +(1-\alpha_t)q_t^{\mathrm{hist}}.
\]
The round-$t$ target is
\[
    \widetilde\pi_t^{(\beta_t)}(z)
    =
    p_{\mathrm{LLM}}(z)r_{\mathrm{PRM}}(z)^{\beta_t}\ind\{\length(z)\le t+1\}.
\]
If $\widetilde\pi_t^{(\beta_t)}(z)>0$, then either $\length(z)\le t$ and $q_t^{\mathrm{hist}}(z)>0$, or $\length(z)=t+1$ and $q_t^{\mathrm{new}}(z)>0$ whenever the LLM assigns positive probability to $z$ through its parent.  Since $0<\alpha_t<1$ and $r_{\mathrm{PRM}}(z)>0$, $\pi_t^{\beta_t}\ll q_t^{\mathrm{mix}}$.

Lemma~\ref{lem:pbs-explicit-weight} gives
\[
    F_t(z)=
    \frac{\widetilde\pi_t^{(\beta_t)}(z)}
    {\widetilde q_t^{\mathrm{mix}}(z)}.
\]
Lemma~\ref{lem:fixed-count-estimator} shows that \eqref{eq:fixed-count-weights-app} is exactly the fixed-count self-normalized mixture-proposal estimator using $N$ samples from $q_t^{\mathrm{new}}$ and $Nt$ samples from $q_t^{\mathrm{hist}}$.  Applying self-normalized importance sampling consistency to the two proposal averages gives
\[
    \frac{\sum_{z\in\Pool_t}W_t(z)f(z)}
    {\sum_{z\in\Pool_t}W_t(z)}
    \xrightarrow[N\to\infty]{p}
    \bbE_{\pi_t^{(\beta_t)}}[f(X)].
\]
This completes the induction.

For adaptive powers, $\sigma_t=\sum_{z\in\Pool_t}a_z^2$ is a bounded continuous function of the empirical weighted measure on the finite state space.  The continuous mapping theorem gives convergence of $\sigma_t$ and therefore of $\beta_{t+1}=\beta_t+\Delta_t$.  Repeating the induction with the limiting adaptive powers proves the adaptive statement.
\end{proof}

\section{Additional Experiments}
\label{app:exp}

\subsection{Experimental Setup}
\label{app:exp_setup}

\paragraph{Benchmark Datasets.} We conduct experiments across four mathematical reasoning benchmarks spanning multiple difficulty levels and problem types, enabling comprehensive assessment of our approach across varying reasoning complexity.

\begin{itemize}
    \item \textbf{MATH500}~\citep{lightman2024lets} comprises 500 competition-level mathematics problems sampled from the MATH dataset~\citep{hendrycks2021measuring}, covering topics including algebra, geometry, number theory, and calculus. The problems are stratified across five difficulty levels, making this benchmark well-suited for evaluating general mathematical reasoning.
    \item \textbf{AMC23}~\citep{amc2023} consists of 40 problems from the 2023 American Mathematics Competition, featuring intermediate-level competition mathematics that bridges the gap between standard curriculum and advanced problem-solving. These problems require creative application of mathematical concepts rather than routine computation.
    \item \textbf{AIME24}~\citep{aime24} contains 30 problems from the 2024 American Invitational Mathematics Examination, representing extremely challenging competition-level mathematics. AIME problems require sophisticated multi-step problem-solving strategies and deep mathematical understanding, making this benchmark discriminative at high performance levels.
    \item \textbf{Minerva Math}~\citep{lewkowycz2022solving} is a collection of advanced STEM reasoning problems spanning mathematics, physics, chemistry, and biology. Unlike purely math benchmarks, Minerva Math evaluates the model's ability to apply mathematical reasoning in broader scientific contexts, including problems that require domain knowledge alongside formal derivation.
    
\end{itemize}

\paragraph{Large Language Models.} We conduct experiments using three state-of-the-art open-source instruction-tuned language models spanning a range of parameter scales, enabling systematic evaluation of our approach across varying computational budgets and model capabilities.

\begin{itemize}
    \item \textbf{Qwen2.5-7B-Instruct}~\citep{qwen2.5} is a 7-billion-parameter instruction-tuned model from the Qwen2.5 family, which has demonstrated strong performance on mathematical reasoning benchmarks relative to its size. Its balance of capability and computational cost makes it a representative mid-scale LLM for inference-time scaling experiments.
    \item \textbf{Qwen2.5-3B-Instruct}~\citep{qwen2.5} is a 3-billion-parameter variant of the same family, providing a lightweight LLM that allows us to assess whether inference-time scaling methods can compensate for reduced model capacity. Evaluating at this scale is particularly relevant under token-efficiency objectives.
    \item 
\textbf{Phi-4-mini-Instruct}~\citep{abouelenin2025phi} is a compact yet capable instruction-tuned model from Microsoft's Phi-4 family, designed to achieve strong reasoning performance at a small parameter footprint. Its inclusion allows us to evaluate our methods on a model from a distinct training lineage, providing evidence of generalization beyond the Qwen model family.
    
\end{itemize}

\paragraph{Process Reward Model.}
We use \textbf{Qwen2.5-Math-PRM-7B}~\cite{zhang2025lessons} as the process reward model throughout our experiments. The PRM provides step-level reward signals for evaluating partial reasoning trajectories during inference and is built on Qwen2.5-Math-7B-Instruct with step-level preference training. We choose this PRM due to its widespread adoption and strong empirical performance in mathematical reasoning.
\paragraph{Baselines.}
We compare against the following inference-time scaling methods:

\begin{itemize}

\item Base model: Single greedy pass without test-time scaling—the minimal compute baseline

\item Self-Consistency \cite{wang2023selfconsistency}: Samples 32 independent trajectories and selects the final answer via majority voting

\item Best-of-N (BoN) \cite{brown2024large}: Generates 32 complete solutions, ranks them by PRM scores of final answers, and selects the highest-scored trajectory

\item Beam Search \cite{snell2025scaling}: Maintains beam width 32 with branching factor 4 and top-8 selection per step using cumulative mean PRM scores.

\item Particle Filtering \cite{puri2025probabilistic} : maintaining 32 particles with systematic resampling triggered when effective sample size drops below threshold

\item MCTS \cite{inoue2025widerdeeperscalingllm}: Monte Carlo tree search using PRM-guided expansion and UCB selection, with final answer determined by tree statistics aggregation

\item DVTS \cite{beeching2024scaling}: Diverse Verified Tree Search—partitions the search budget into $M$ independent subtrees, each explored via beam search with width $N/M$. We use $M=4$ subtrees with beam width 8 per subtree (total budget $N=32$)

\end{itemize}

\paragraph{PRM scoring and search budgets.}
For every prefix, we use the PRM score of the last generated reasoning step as the prefix score $r_{\mathrm{PRM}}(z)$; final answers are selected by the best PRM score among complete solutions.  Unless otherwise stated, all methods use a fixed step horizon $T=30$ following the step-level reasoning setup of \cite{feng2025stepbystep}.

\paragraph{Implementation Details.} 
We evaluate all baselines with $N=32$ candidate expansions per step, while our proposed methods use $N=8$ to demonstrate computational efficiency. All sampling-based methods use temperature 0.7. For baseline configurations, we follow their original settings: Beam Search and DVTS employ greedy selection with PRM-based pruning at each decoding step.

For our proposed methods, \textbf{PB-SMC} maintains $N=8$ particles with step-wise PRM scoring for resampling weights. \textbf{SPS} operates with budget $N=8$, employing adaptive subpool expansion based on mean PRM scores and dynamic pruning strategies. All methods use \textbf{Qwen2.5-Math-7B-PRM} as the default process reward model.

All reported results are averaged over 5 random seeds for AMC23 and AIME24, and 3 seeds for Minerva and MATH500, with standard deviations reported where applicable.

\subsection{Extended Main Results}
\label{sec:full_results}
\begin{table*}[t]
\vspace{-3mm}
\centering
\scriptsize
\caption{%
    \textbf{Comparison of test time scaling methods across benchmarks and models.} 
    We report accuracy (Acc, $\uparrow$ better) and token usage (\#Tokens, $\downarrow$ better) across four benchmarks.
    Results include mean and standard deviation over 5 random seeds for AMC23 vs AIME24 and 3 random seeds for Minerva vs MATH500.
    \textbf{Bold} indicates best performance; \underline{underline} indicates runner-up.
}
\label{tab:full_table}
\resizebox{\linewidth}{!}{%
\begin{tabular}{l c cc cc cc cc cc}
\toprule
& & \multicolumn{2}{c}{\textbf{Minerva}}
& \multicolumn{2}{c}{\textbf{Math500}}
& \multicolumn{2}{c}{\textbf{AMC23}}
& \multicolumn{2}{c}{\textbf{AIME24}}
& \multicolumn{2}{c}{\textbf{Average}} \\
\cmidrule(lr){3-4}\cmidrule(lr){5-6}\cmidrule(lr){7-8}\cmidrule(lr){9-10}\cmidrule(lr){11-12}
\textbf{Method} & \textbf{N} &
\textbf{Accuracy\,$\uparrow$} & \textbf{\#Tokens\,$\downarrow$} &
\textbf{Accuracy\,$\uparrow$} & \textbf{\#Tokens\,$\downarrow$} &
\textbf{Accuracy\,$\uparrow$} & \textbf{\#Tokens\,$\downarrow$} &
\textbf{Accuracy\,$\uparrow$} & \textbf{\#Tokens\,$\downarrow$} &
\textbf{Accuracy\,$\uparrow$} & \textbf{\#Tokens\,$\downarrow$} \\
\midrule

\multicolumn{12}{c}{\textbf{Qwen2.5-7B-Instruct}} \\
\cmidrule{1-12}

Base Model & 1 & \sd{34.68}{1.12} & 617 & \sd{76.00}{0.80} & 584 & \sd{54.50}{2.92} & 874 & \sd{11.33}{1.63} & 1039 & 44.13 & 779 \\
\addlinespace[2pt]

Pass@\textit{k}$^\dagger$ & 32 & \sd{52.33}{0.21} & 19765 & \sd{93.67}{0.12} & 18788 & \sd{90.50}{2.45} & 27979 & \sd{31.33}{3.40} & 33262 & 66.96 & 24949 \\
\addlinespace[2pt]

Self-Consistency & 32 & \sd{38.85}{0.56} & 19765 & \sd{83.27}{0.31} & 18788 & \sd{62.00}{1.87} & 27979 & \sd{16.67}{0.00} & 33262 & 50.20 & 24949 \\
\addlinespace[2pt]

Best-of-N & 32 & \sd{38.48}{0.42} & 19765 & \sd{86.40}{0.20} & 18788 & \sd{65.00}{1.58} & 27564 & \textbf{\sd{20.67}{1.33}} & 33344 & 52.64 & 24865 \\
\addlinespace[2pt]

Beam Search & 32 & \sd{40.20}{0.56} & 23900 & \sd{86.27}{0.31} & 22771 & \sd{62.00}{2.92} & 39195 & \sd{20.00}{2.11} & 55150 & 52.12 & 35254 \\
\addlinespace[2pt]

DVTS & 32 & \sd{39.22}{0.43} & 23549 & \sd{86.73}{0.23} & 22978 & \sd{63.00}{1.00} & 34670 & \sd{16.67}{2.11} & 46725 & 51.41 & 31981 \\
\addlinespace[2pt]

MCTS & 32 & \sd{36.27}{0.87} & 25208 & \sd{78.47}{0.81} & 22006 & \sd{58.00}{1.00} & 30328 & \sd{16.67}{1.33} & 35644 & 47.35 & 28297 \\
\addlinespace[2pt]

SMC & 32 & \sd{40.32}{0.94} & 21642 & \sd{86.67}{0.41} & 23307 & \sd{62.00}{3.54} & 35546 & \sd{18.65}{2.66} & 41031 & 51.91 & 30382 \\
\addlinespace[2pt]

\cmidrule{1-12}
\addlinespace[1pt]

\textbf{PB-SMC (ours)} & 8 & \underline{\sd{41.79}{0.56}} & \underline{13216} & \underline{\sd{87.47}{0.31}} & \underline{10628} & \underline{\sd{66.50}{2.85}} & \textbf{17637} & \sd{19.33}{1.50} & \underline{27581} & \underline{53.77} & \underline{17266} \\
\addlinespace[2pt]


\textbf{SPS (ours)} & \textbf{8} & \textbf{\sd{43.26}{0.77}} & \textbf{9387} & \textbf{\sd{87.87}{0.70}} & \textbf{9076} & \textbf{\sd{67.50}{2.50}} & \underline{22446} & \underline{\sd{20.00}{2.37}} & \textbf{24347} & \textbf{54.66} & \textbf{16314} \\
\addlinespace[2pt]

\midrule

\multicolumn{12}{c}{\textbf{Qwen2.5-3B-Instruct}} \\
\cmidrule{1-12}

Base Model & 1 & \sd{29.53}{1.29} & 617 & \sd{64.47}{0.90} & 607 & \sd{36.50}{4.06} & 839 & \sd{6.00}{2.49} & 1046 & 34.13 & 777 \\
\addlinespace[2pt]

Pass@\textit{k}$^\dagger$ & 32 & \sd{46.20}{0.21} & 19765 & \sd{91.73}{0.12} & 19427 & \sd{87.50}{2.74} & 26863 & \sd{26.67}{3.65} & 33477 & 63.03 & 24883 \\
\addlinespace[2pt]

Self-Consistency & 32 & \sd{33.21}{0.21} & 19765 & \sd{77.87}{0.50} & 19427 & \sd{54.00}{3.39} & 26863 & \sd{14.67}{1.63} & 33477 & 44.94 & 24883 \\
\addlinespace[2pt]

Best-of-N & 32 & \sd{33.70}{0.42} & 19765 & \sd{81.20}{0.40} & 19427 & \sd{61.50}{1.22} & 26863 & \sd{16.00}{1.33} & 33477 & 48.10 & 24883 \\
\addlinespace[2pt]

Beam Search & 32 & \sd{35.78}{0.57} & 23131 & \underline{\sd{82.27}{0.31}} & 25322 & \sd{60.50}{1.87} & 41169 & \sd{14.00}{1.33} & 60382 & 48.14 & 37501 \\
\addlinespace[2pt]

DVTS & 32 & \sd{35.05}{0.21} & 24971 & \sd{81.80}{0.20} & 26314 & \textbf{\sd{65.50}{2.92}} & 42538 & \sd{13.33}{2.11} & 57349 & 48.92 & 37793 \\
\addlinespace[2pt]

MCTS & 32 & \sd{33.58}{0.35} & 26121 & \sd{72.00}{0.98} & 22658 & \sd{56.50}{1.22} & 31825 & \sd{15.33}{2.13} & 34589 & 44.35 & 28798 \\
\addlinespace[2pt]

SMC & 32 & \sd{36.52}{0.75} & 25143 & \sd{81.33}{0.41} & 24362 & \sd{59.50}{1.87} & 43715 & \sd{14.65}{2.66} & 52562 & 48.00 & 36446 \\
\addlinespace[2pt]

\cmidrule{1-12}
\addlinespace[1pt]

\textbf{PB-SMC (ours)} & 8 & \underline{\sd{38.23}{0.64}} & \underline{15450} & \sd{82.07}{0.50} & \underline{11727} & \sd{62.50}{2.50} & \underline{20160} & \underline{\sd{16.67}{2.11}} & \underline{29108} & \underline{49.87} & \underline{19111} \\
\addlinespace[2pt]


\textbf{SPS (ours)} & \textbf{8} & \textbf{\sd{38.73}{0.93}} & \textbf{10800} & \textbf{\sd{83.20}{0.60}} & \textbf{9359} & \underline{\sd{63.00}{1.87}} & \textbf{16644} & \textbf{\sd{17.33}{2.49}} & \textbf{25130} & \textbf{50.57} & \textbf{15483} \\
\addlinespace[2pt]

\midrule

\multicolumn{12}{c}{\textbf{Phi-4-mini-Instruct}} \\
\cmidrule{1-12}

Base Model & 1 & \sd{23.90}{1.11} & 535 & \sd{62.80}{0.80} & 525 & \sd{42.50}{3.54} & 781 & \sd{7.33}{3.27} & 948 & 34.13 & 697 \\
\addlinespace[2pt]

Pass@\textit{k}$^\dagger$ & 32 & \sd{54.53}{0.21} & 17119 & \sd{91.07}{0.12} & 16774 & \sd{79.00}{3.39} & 25003 & \sd{28.67}{5.81} & 30316 & 63.32 & 22303 \\
\addlinespace[2pt]

Self-Consistency & 32 & \sd{33.70}{0.21} & 17119 & \sd{77.00}{0.40} & 16774 & \sd{53.50}{2.55} & 25003 & \sd{13.33}{0.00} & 30316 & 44.38 & 22303 \\
\addlinespace[2pt]

Best-of-N & 32 & \sd{34.80}{0.42} & 17119 & \sd{80.80}{0.20} & 16774 & \sd{50.50}{4.30} & 25003 & \sd{14.67}{2.67} & 30316 & 45.19 & 22303 \\
\addlinespace[2pt]

Beam Search & 32 & \sd{33.46}{0.37} & 23768 & \sd{81.20}{0.20} & 21397 & \sd{56.50}{4.06} & 37513 & \underline{\sd{16.00}{3.89}} & 39635 & 46.79 & 30578 \\
\addlinespace[2pt]

DVTS & 32 & \sd{33.82}{0.74} & 22140 & \sd{81.53}{0.12} & 20956 & \sd{55.00}{2.24} & 36815 & \underline{\sd{16.00}{3.27}} & 37532 & 46.59 & 29361 \\
\addlinespace[2pt]

MCTS & 32 & \sd{32.75}{2.08} & 23108 & \sd{73.20}{0.56} & 28036 & \sd{55.00}{1.50} & 33320 & \sd{14.67}{1.66} & 30316 & 43.91 & 28695 \\
\addlinespace[2pt]

SMC & 32 & \sd{34.44}{0.75} & 24907 & \sd{81.40}{0.59} & 25459 & \sd{54.50}{3.32} & 42574 & \sd{14.00}{0.98} & 45448 & 46.09 & 34597 \\
\addlinespace[2pt]

\cmidrule{1-12}
\addlinespace[1pt]

\textbf{PB-SMC (ours)} & 8 & \textbf{\sd{37.87}{0.37}} & \underline{13198} & \underline{\sd{82.07}{0.50}} & \underline{10422} & \textbf{\sd{57.50}{3.06}} & \underline{15547} & \textbf{\sd{17.33}{3.89}} & \underline{21258} & \textbf{48.69} & \underline{15106} \\
\addlinespace[2pt]


\textbf{SPS (ours)} & \textbf{8} & \underline{\sd{35.42}{0.56}} & \textbf{9832} & \textbf{\sd{82.67}{0.50}} & \textbf{7959} & \underline{\sd{57.00}{2.92}} & \textbf{13497} & \underline{\sd{16.00}{1.33}} & \textbf{15056} & \underline{47.77} & \textbf{11586} \\
\addlinespace[2pt]


\bottomrule
\end{tabular}%
} 
\end{table*}

This section presents the full experimental results (Table \ref{tab:full_table}) across all three LLMs, extending the main-paper analysis (Table~\ref{tab:main_search}) with the additional model omitted due to space constraints.
Overall, the trends observed in the main paper hold consistently across all settings.

Across all three LLMs and four benchmarks, both PB-SMC and SPS consistently outperform strong inference-time search baselines while using substantially fewer tokens.
In all cases, our methods achieve competitive or superior accuracy using only $N=8$ candidates, compared to $N=32$ required by all baselines.
When averaged across all model--benchmark pairs, SPS yields a 49.4\% reduction in token usage while improving accuracy, and PB-SMC achieves a 36.7\% token reduction with improved or comparable accuracy.

On Qwen2.5-7B-Instruct, SPS performs best overall, achieving the highest average accuracy while significantly reducing token consumption.
This advantage is particularly pronounced on Minerva and AMC23, where SPS attains strong gains despite maintaining a fourfold reduction in candidate count.
We attribute this behavior to the use of Qwen2.5-Math-7B-PRM as the preference reward model, which is a natural fit for the base LLM and likely provides more informative guidance.

On Phi-4-mini-instruct, PB-SMC emerges as the most reliable method, achieving the highest aggregate accuracy across benchmarks.
While SPS remains competitive, PB-SMC demonstrates more stable performance across tasks, reflecting its conservative but robust exploration strategy.
This robustness is consistently observed across all benchmarks, including more challenging settings such as AIME24.

On the third LLM, we observe the same qualitative trends as in the other two settings.
Both PB-SMC and SPS outperform all baselines in the accuracy--efficiency trade-off regime, maintaining strong performance at a fraction of the token cost.
PB-SMC again exhibits stable scaling behavior across benchmarks, while SPS tends to achieve higher peak accuracy when the PRM is well aligned with the underlying LLM.

Taken together, these results demonstrate that the improvements of PB-SMC and SPS are not model-specific but generalize across diverse LLMs.
Moreover, they highlight a consistent accuracy--efficiency advantage over existing inference-time search methods, even under severely constrained sampling budgets.

\subsection{Experimental Configuration for Model Serving}
\label{app:serving_config}

\paragraph{Hardware and Serving Infrastructure.}
We deploy the LLM and process reward model on separate NVIDIA A100 GPUs using vLLM framework for efficient batched inference. This configuration enables parallel evaluation of reasoning trajectories and step-level reward computation without memory contention. Table~\ref{tab:serving_config} details the serving configuration for each model.

\begin{table}[t]
\centering
\caption{vLLM serving configuration for LLMs and reward models. Each model is allocated a dedicated A100 GPU to enable efficient parallel inference.}
\vspace{2mm}
\label{tab:serving_config}
\small
\setlength{\tabcolsep}{4pt}
\begin{tabular}{lcccc}
\toprule
\textbf{Configuration} 
& \makecell{\textbf{Qwen2.5-7B}\\\textbf{Instruct}} 
& \makecell{\textbf{Qwen2.5-Math}\\\textbf{PRM-7B}} 
& \makecell{\textbf{Qwen2.5-3B}\\\textbf{Instruct}} 
& \makecell{\textbf{Phi-4-mini}\\\textbf{Instruct}} \\
\midrule
GPU Allocation           & A100 & A100 & A100 & A100 \\
Tensor Parallel Size     & 1    & 1    & 1    & 1    \\
Max Model Length         & 4096 & 4096 & 4096 & 4096 \\
GPU Memory Utilization   & 0.9  & 0.9  & 0.9  & 0.9  \\
\bottomrule
\end{tabular}
\end{table}


\paragraph{Frontier vs.\ Backtracking Test-Time Scaling.}
Common test-time scaling baselines are \emph{frontier} methods that only expand forward from the current prefix. Best-of-$N$ and Self-Consistency generate multiple complete reasoning trajectories via sampling and then select/aggregate a final answer, without revising earlier prefix decisions. 
Beam Search and verifier-guided variants such as DVTS similarly maintain a set of partial hypotheses at the decoding frontier and apply step-wise expand-and-prune, advancing the frontier rather than backtracking to earlier steps. 


In contrast, our method is an \emph{backtracking} algorithm: it can revert to an earlier step $i<t$ and branch from that prefix to refine the solution. 

\subsection{Runtime Comparison}
\label{app:runtime}

Figure~\ref{fig:runtime_comparison} presents wall-clock runtime measurements for all evaluated methods on the MATH500 benchmark. Measurements are conducted on the same hardware configuration described in Section~\ref{app:serving_config}, with runtime reported as the average time to solve a single problem across all 500 instances.

\begin{figure}[h]
    \centering
    \begin{subfigure}[b]{0.48\textwidth}
        \centering
        \includegraphics[width=\textwidth]{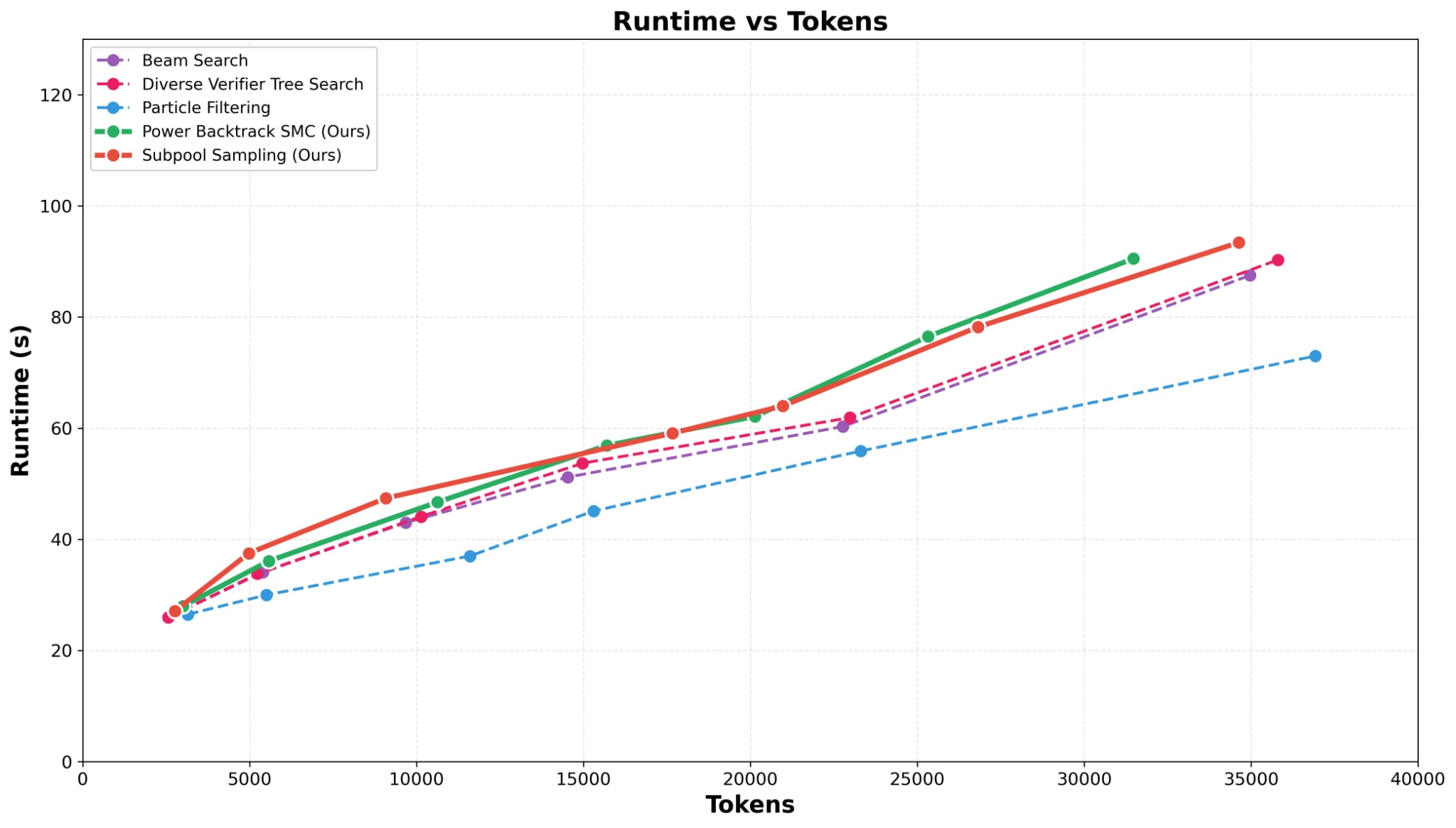}
        \caption{Runtime vs Tokens}
        \label{fig:runtime_only}
    \end{subfigure}
    \hfill
    \begin{subfigure}[b]{0.48\textwidth}
        \centering
        \includegraphics[width=\textwidth]{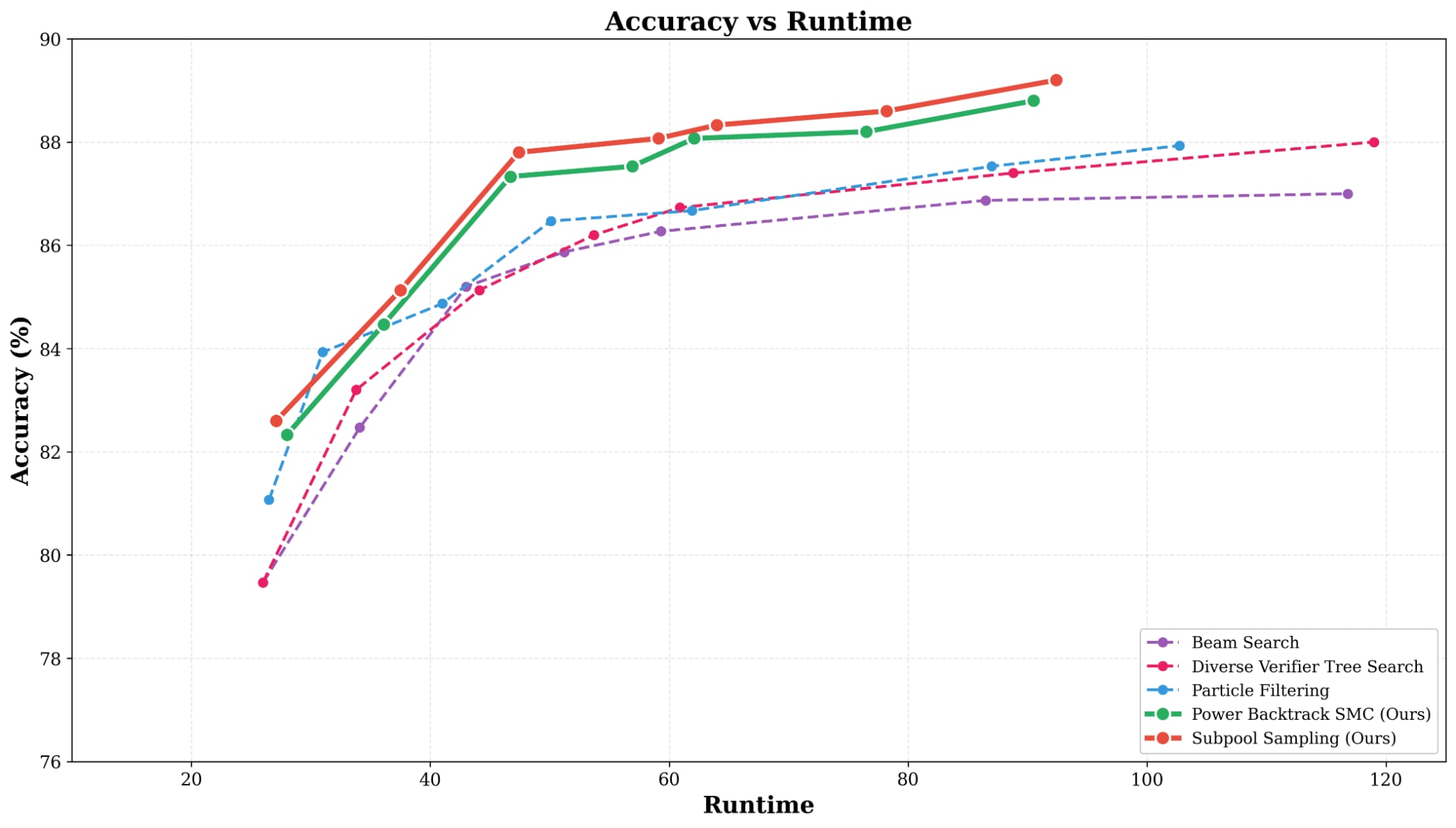}
        \caption{Runtime vs Accuracy}
        \label{fig:acc_runtime}
    \end{subfigure}
    \caption{Runtime comparison across test-time scaling methods on MATH500. Our proposed methods (PB-SMC and SPS) scale linearly in the number of tokens, and carry only a small runtime overhead of roughly 10\% in comparison with the nearest competitive baselines (beam search, DVTS). Taking into account the token count advantage of our methods (3X), this overhead becomes negligible.}
    \label{fig:runtime_comparison}
\end{figure}

As shown in Figure~\ref{fig:runtime_comparison}, methods that generate $N=32$ independent samples (Self-Consistency, Best-of-N) exhibit the highest runtime due to sequential generation without pruning. Beam Search and DVTS achieve moderate runtime by leveraging early pruning, though they still maintain large beam widths. Our proposed methods, operating at $N=8$, demonstrate substantial runtime improvements---approximately $2.5\times$ faster than Beam Search---while achieving competitive accuracy as reported in Table~\ref{tab:main_search}. Notably, the runtime of PB-SMC and SPS remains comparable despite their different search strategies, as both methods maintain similar numbers of active candidates at each step.

The runtime efficiency of our methods stems from two key factors: (1) reduced candidate set size ($N=8$ vs $N=32$), which directly reduces the number of forward passes through the LLM, and (2) efficient resampling and pruning strategies that avoid generating low-quality trajectories. Combined with the memory efficiency discussed in Section~\ref{app:serving_config}, these results demonstrate that our methods achieve a favorable trade-off between computational cost and reasoning performance.

\section{Ablation Study}
\label{sec:ablation}

We conduct ablations to isolate the effect of historical backtracking and the stochastic mechanisms used to access the persistent pool. Unless otherwise stated, all experiments use Qwen2.5-7B-Instruct as the base generator and Qwen2.5-Math-PRM-7B as the process reward model. We study two questions:
(i) whether allowing search to revisit historical prefixes improves over frontier-only selection, and
(ii) how robust SPS is to the backtrack ratio.

\subsection{Hyperparameter Sensitivity}
\label{sec:hyperparam_sensitivity}

We study the sensitivity of the two proposed backtracking mechanisms to their main
search hyperparameters. For SPS, the default algorithm does not use a fixed
backtrack ratio. Instead, it uses the adaptive subpool ratio introduced in
Section~\ref{sec:methodology}, where the ratio is determined by the average PRM
score of the current persistent pool. We therefore treat fixed-ratio SPS as a
diagnostic variant: it is used only to understand how SPS behaves as the amount of
stochastic backtracking is varied, not to tune the main SPS results. For PB-SMC,
we study the two schedule hyperparameters defined in
Appendix~\ref{app:pbsmc-schedules}: the power-step parameter $\gamma$ in ~\cref{eq:pbsmc-beta-schedule-app}, which controls how quickly the powered PRM
target sharpens through the adaptive sequence $\{\beta_t\}$, and the mixture
parameter $g_{\min}$ in Eq.~\eqref{eq:pbsmc-alpha-schedule-app}, which controls
the terminal history-to-new ratio $g_t=(1-\alpha_t)/\alpha_t$ and hence the
mixture probability $\alpha_t=1/(1+g_t)$. Smaller $g_{\min}$ places more mass on
newly generated children near the end of search and makes PB-SMC closer to
frontier-only SMC, while larger $g_{\min}$ retains more mass on historical
particles and increases stochastic backtracking.

\subsubsection{Sensitivity to SPS subpool ratio}
\label{sec:ablation_ratio}

The SPS method used in our main experiments employs an adaptive subpool ratio rather than a fixed ratio.
To isolate the effect of the subpool ratio, we also evaluate fixed-ratio variants
on Math500. In this sweep, the adaptive rule for $\rho_t$ is disabled and replaced
by a constant ratio in $\{0.1,0.2,\ldots,0.9\}$. The fixed-ratio results should
therefore be interpreted as a sensitivity analysis, not as the default SPS
configuration. The adaptive-ratio result, corresponding to the SPS algorithm used
in the main experiments, is reported separately in the ``Adaptive'' column.

\begin{table}[t]
\centering
\scriptsize
\caption{%
    \textbf{Sensitivity to SPS subpool ratio on Math500.}
    We report accuracy (Acc, \%, $\uparrow$ better) for three base models.
    Fixed ratios are used only for sensitivity analysis. The default SPS method
    uses the adaptive ratio
    $\rho_t=|\Pool_{t-1}|^{-1}\sum_{z\in\Pool_{t-1}}r_{\mathrm{PRM}}(z)$, reported in the
    ``Adaptive'' column. The final column reports mean $\pm$ standard deviation
    across fixed ratios only. \textbf{Bold} indicates the best fixed-ratio result
    for each model.
}
\vspace{2mm}
\label{tab:ablation_ratio_all}
\renewcommand\tabcolsep{2.5pt}
\renewcommand\arraystretch{1.1}
\begin{tabular}{l ccccccccc c c}
\toprule
\textbf{Model}
& \textbf{0.1} & \textbf{0.2} & \textbf{0.3} & \textbf{0.4} & \textbf{0.5}
& \textbf{0.6} & \textbf{0.7} & \textbf{0.8} & \textbf{0.9}
& \textbf{Adaptive}
& \textbf{Mean $\pm$ Std} \\
\midrule
Qwen2.5-7B-Instruct
& 86.0 & 87.6 & \textbf{88.0} & 86.4 & 86.2 & 87.0 & 87.2 & 87.4 & 87.4
& 87.8
& $87.02 \pm 0.64$ \\
\addlinespace[2pt]
Qwen2.5-3B-Instruct
& 82.4 & 83.0 & 82.8 & 82.8 & 82.6 & 82.8 & \textbf{83.2} & 83.0 & 82.6
& 83.2
& $82.80 \pm 0.23$ \\
\addlinespace[2pt]
Phi-4-mini-instruct
& 81.6 & 82.2 & 82.0 & \textbf{82.6} & 82.2 & 82.4 & 82.0 & 82.4 & 82.2
& 82.6
& $82.18 \pm 0.27$ \\
\bottomrule
\end{tabular}
\end{table}

Table~\ref{tab:ablation_ratio_all} shows that SPS is not highly sensitive to the
fixed subpool ratio. Across fixed ratios, accuracy varies by only $2.0$ points for
Qwen2.5-7B-Instruct, $0.8$ points for Qwen2.5-3B-Instruct, and $1.0$ point for
Phi-4-mini-instruct. The standard deviations across the fixed-ratio sweep are also
small: $0.64$, $0.23$, and $0.27$, respectively. Thus SPS remains stable across a
broad range of stochasticity levels.

The adaptive-ratio default is competitive with the best fixed ratios on all three
models, achieving $87.8\%$, $83.2\%$, and $82.6\%$ accuracy for
Qwen2.5-7B-Instruct, Qwen2.5-3B-Instruct, and Phi-4-mini-instruct, respectively.
This supports the use of the adaptive ratio in the main experiments: it avoids
introducing an additional fixed-ratio hyperparameter while preserving the benefit
of stochastic access to historical prefixes. Intuitively, the fixed-ratio sweep
shows that SPS works over a wide range of exploration levels, while the adaptive
rule provides a simple automatic mechanism for choosing the exploration level from
the current pool quality.

\subsubsection{Sensitivity to PB-SMC power and mixture schedules}
\label{sec:ablation_pbsmc_hyperparams}

We next study the two main schedule hyperparameters of PB-SMC. The first is the
power-step parameter $\gamma$ in Eq.~\eqref{eq:pbsmc-beta-schedule-app}. This
parameter controls the rate at which the powered PRM target sharpens through the
adaptive sequence $\{\beta_t\}$. Larger values of $\gamma$ increase the influence
of high PRM scores more aggressively and make PB-SMC more exploitative, whereas
smaller values keep the sampling distribution more diffuse. The second is
$g_{\min}$ in Eq.~\eqref{eq:pbsmc-alpha-schedule-app}, the terminal value of the
history-to-new ratio $g_t=(1-\alpha_t)/\alpha_t$. Since
$\alpha_t=1/(1+g_t)$, smaller $g_{\min}$ makes $\alpha_t$ closer to one near the
end of search and emphasizes newly generated children. Larger $g_{\min}$ assigns
more mixture mass to retained historical particles, increasing the amount of
stochastic backtracking.

To avoid tuning on the reported test benchmarks, we construct a held-out validation
set of $500$ problems sampled from the MATH training split. This validation set is
used only to choose global PB-SMC hyperparameters. The selected values are then
fixed for all main experiments across datasets and models. We perform a sequential
one-dimensional sweep over the schedules in Appendix~\ref{app:pbsmc-schedules}:
we first fix $g_{\min}=0.5$ and vary $\gamma$; after selecting $\gamma=9$, we fix
$\gamma=9$ and vary $g_{\min}$. This validation procedure selects $\gamma=9$ and
$g_{\min}=0.4$, which are used as the default PB-SMC hyperparameters in all main
experiments.

\begin{table}[t]
\vspace{2mm}
\centering
\scriptsize
\caption{%
    \textbf{Sensitivity of PB-SMC hyperparameters on a held-out MATH validation set.}
    We report accuracy (Acc, \%, $\uparrow$ better) for Qwen2.5-7B-Instruct on
    $500$ validation problems sampled from the MATH training split. For the
    $\gamma$ sweep, we fix $g_{\min}=0.5$. For the $g_{\min}$ sweep, we fix
    $\gamma=9$. The final column reports mean $\pm$ standard deviation across the
    sweep. \textbf{Bold} indicates best performance.
}
\vspace{2mm}
\label{tab:pbsmc_hyperparam_sensitivity}
\renewcommand\tabcolsep{2.5pt}
\renewcommand\arraystretch{1.1}

\begin{tabular}{l cccccccccccc}
\toprule
\multicolumn{13}{c}{\textbf{Power schedule: varying $\gamma$ with $g_{\min}=0.5$}} \\
\midrule
$\gamma$
& 1 & 3 & 5 & 7 & 9 & 11 & 13 & 15 & 17 & 19 & 21
& Mean $\pm$ Std \\
\midrule
Acc
& 85.6 & 84.8 & 85.2 & 85.6 & \textbf{86.0} & 85.4 & 85.2 & 85.4 & 85.0 & 84.6 & 84.2
& $85.18 \pm 0.49$ \\
\midrule
\multicolumn{13}{c}{\textbf{Mixture schedule: varying $g_{\min}$ with $\gamma=9$}} \\
\midrule
$g_{\min}$
& 0.01 & 0.05 & 0.1 & 0.2 & 0.3 & 0.4 & 0.5 & 0.6 & 0.7 & 0.8 & 0.9
& Mean $\pm$ Std \\
\midrule
Acc
& 83.2 & 84.6 & 85.4 & 85.8 & 86.2 & \textbf{86.4} & 86.0 & 85.6 & 85.0 & 84.2 & 83.4
& $85.07 \pm 1.05$ \\
\bottomrule
\end{tabular}
\vspace{-1em}
\end{table}

Table~\ref{tab:pbsmc_hyperparam_sensitivity} shows that PB-SMC is reasonably
robust to the power-step parameter $\gamma$ in
Eq.~\eqref{eq:pbsmc-beta-schedule-app}. Across the full sweep
$\gamma\in\{1,3,\ldots,21\}$, accuracy varies by only $1.8$ points. Over the broad
middle range $\gamma\in[5,15]$, accuracy remains between $85.2\%$ and $86.0\%$,
a range of only $0.8$ points. The best validation accuracy is obtained at
$\gamma=9$. Very small values of $\gamma$ under-sharpen the powered PRM
distribution, making PB-SMC less able to exploit high-quality prefixes. Very large
values, such as $\gamma=19$ and $\gamma=21$, slightly degrade performance,
consistent with over-sharpening the selection distribution and increasing the risk
of premature concentration on a small number of high-scoring prefixes.

The mixture parameter $g_{\min}$ in Eq.~\eqref{eq:pbsmc-alpha-schedule-app}
exhibits a clearer exploration--exploitation trade-off. When $g_{\min}$ is very
small, such as $0.01$ or $0.05$, the terminal mixture probability
$\alpha_t=1/(1+g_t)$ becomes close to one, so PB-SMC places most of its mass on
newly generated children near the end of search. In this regime, the method becomes
closer to frontier-only SMC and loses part of the benefit of stochastic
backtracking. Conversely, when $g_{\min}$ is very large, such as $0.8$ or $0.9$,
PB-SMC retains too much mass on historical particles, which can spend compute
revisiting older prefixes rather than extending promising current ones. The best
validation accuracy is achieved at $g_{\min}=0.4$, and performance remains stable
over the middle range $g_{\min}\in[0.2,0.6]$, where accuracy stays within $0.8$
points of the best value.

Overall, the PB-SMC sensitivity results support the same conclusion as the SPS
ratio sweep: stochastic backtracking does not depend on delicate hyperparameter
tuning. PB-SMC performs best with a moderate power schedule and a moderate mixture
between newly generated children and retained historical prefixes. We therefore
use $\gamma=9$ and $g_{\min}=0.4$ as the default PB-SMC hyperparameters in all
main experiments.



\section{Broader Impacts}
\label{app:impacts}

Our work proposes a test-time search procedure that reallocates inference-time compute more efficiently for LLM reasoning by revisiting historical prefixes, improving the accuracy--token trade-off on public mathematical reasoning benchmarks. 
A potential positive impact is reduced computational cost (and associated energy use) for achieving a target level of reasoning accuracy, which can make reliable reasoning systems more accessible in resource-constrained settings (e.g., smaller deployments or limited inference budgets). 
Because the approach is test-time only and does not require additional training, it may lower the barrier to improving the performance of existing models in educational, scientific, and engineering workflows where correctness and efficiency are important. 

At the same time, improving token-efficiency and reasoning performance can increase the practical capability of LLM-based systems, which may enable misuse in contexts such as accelerating the production of convincing but incorrect outputs when the underlying verifier signals are imperfect. In addition, while our method is designed to be token-efficient, broader deployment could still increase overall usage and total compute, depending on how it is adopted. We also do not introduce or release new datasets or model checkpoints, and evaluate only on publicly available benchmarks.



\end{document}